\documentclass[final,5p,times,twocolumn]{elsarticle}

\usepackage{graphicx}

\usepackage{amssymb}
\usepackage{amsthm}

\usepackage{lineno}

\usepackage{amsmath}
\usepackage{epstopdf}
\usepackage{subfigure}
\usepackage{multirow}
\usepackage{algorithm}
\usepackage{algorithmic}
\usepackage{hyperref}
\newtheorem{theorem}{Theorem}
\newtheorem{lemma}{Lemma}

\biboptions{square, numbers, sort, compress}

\journal{Neurocomputing}

\begin{document}

\begin{frontmatter}



\title{Symmetric low-rank representation for subspace clustering}


\author{Jie Chen}
\author{Haixian Zhang}
\author{Hua Mao}
\author{Yongsheng Sang}
\author{Zhang Yi \corref{cor1}}
\ead{zhangyi@scu.edu.cn; zyiscu@gmail.com}

\cortext[cor1]{Corresponding author}

\address{Machine Intelligence Laboratory, College of Computer Science, Sichuan University, Chengdu 610065, P. R. China}

\begin{abstract}

We propose a symmetric low-rank representation (SLRR) method for subspace clustering, which assumes that a data set is approximately drawn from the union of multiple subspaces. The proposed technique can reveal the membership of multiple subspaces through the self-expressiveness property of the data. In particular, the SLRR method considers a collaborative representation combined with low-rank matrix recovery techniques as a low-rank representation to learn a symmetric low-rank representation, which preserves the subspace structures of high-dimensional data. In contrast to performing iterative singular value decomposition in some existing low-rank representation based algorithms, the symmetric low-rank representation in the SLRR method can be calculated as a closed form solution by solving the symmetric low-rank optimization problem. By making use of the angular information of the principal directions of the symmetric low-rank representation, an affinity graph matrix is constructed for spectral clustering. Extensive experimental results show that it outperforms state-of-the-art subspace clustering algorithms.

\end{abstract}


\end{frontmatter}


\section{Introduction}
\label{Introduction}

Subspace clustering is one of the fundamental topics in machine learning, computer vision, and pattern recognition, e.g., image representation \cite{Eldar2009RRS, Liu2010LRR}, face clustering \cite{Liu2010LRR1, Elhamifar2013SSC, Liu2010LRR}, and motion segmentation \cite{Rao2008MS, Lauer2009SC, Rao2010MotionSeg, Aldroubi2012MS, Vidala2013LRSC}. The importance of subspace clustering is evident in the vast amount of literature thereon, because it is a crucial step in inferring structure information of data from subspaces through data analysis \cite{Vidal2010SC, Sim2013SC, BMcWilliams2014SC}. Subspace clustering refers to the problem of clustering samples drawn from the union of low-dimensional subspaces, into their subspaces.

When considering subspace clustering in various applications, several types of available visual data are high-dimensional, such as digital images, video surveillance, and traffic monitoring. These high-dimensional data often have a small intrinsic dimension, which is often much smaller than the dimension of the ambient space. For instance, face images of a subject, handwritten images of a digit with different rotations, and feature trajectories of a moving object in a video often lie in a low-dimensional subspace of the ambient space \cite{Boult1991Factor, Basri2003}. To describe a given collection of data well, a more general model is to consider data drawn from the union of low-dimensional subspaces instead of a single lower-dimensional subspace \cite{Liu2010LRR, Dyer2013SC}.

Subspace clustering has been studied extensively over several decades. A number of techniques for exploiting low-dimensional structures of high-dimensional data have been proposed to tackle subspace clustering. Based on their underlying techniques, subspace clustering methods can be roughly divided into four categories according to the mechanism used: algebraic \cite{Vidal2005GPCA}, statistical \cite{Fischler1981RANSAC}, iterative \cite{Ho2003KSC}, and spectral clustering based methods \cite{Liu2010LRR, Ni2010LRRPSD, Zhuang2012NLRR, Peng2012L2Graph, Elhamifar2013SSC}. For a more detailed explanation of these algorithms, we refer the reader to \cite{Vidal2010SC}, which contains a recent review.

\begin{figure*}[!htbp]
\centering
\label{fig:example}
\includegraphics[width=0.9\textwidth]{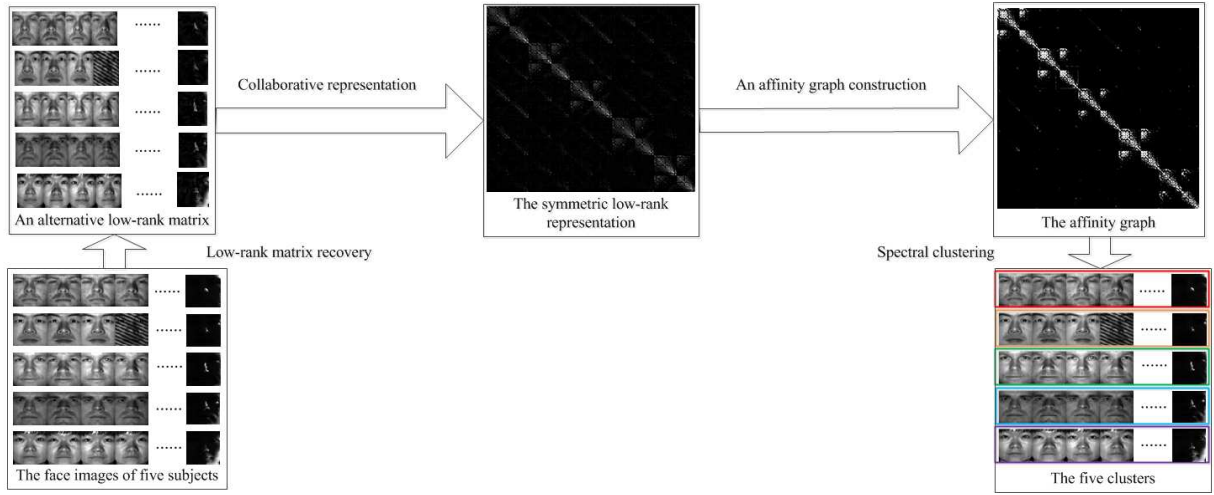}
\caption{Illustration of the clustering problem with five subjects.}
\label{fig:graphicalabstracts} 
\end{figure*}

If there are no errors in the data, i.e., the data are strictly drawn from multiple subspaces, several existing methods can be used to solve subspace clustering exactly \cite{Costeira1998MF, Wei2011RSI, Liu2010LRR, Elhamifar2013SSC}. However, the assumption of low-dimensional intrinsic structures of data is often violated when the real observations are contaminated by noise and gross corruption. Consequently, this results in inferior performance. A number of research efforts have focused on these problems. Spectral clustering based methods, such as sparse representation \cite{Elhamifar2013SSC}, low-rank representation \cite{Liu2010LRR}, and their extensions \cite{Ni2010LRRPSD, Liu2012FRR, Zhuang2012NLRR, Liu2013MF, Zhang2014FLRR} have yielded excellent performance in exploiting low-dimensional structures of high-dimensional data. Most existing methods perform subspace clustering involving two steps: first, learning an affinity matrix that encodes the subspace memberships of samples, and then obtaining the final clustering results with the learned affinity matrix using spectral clustering algorithms such as normalized cuts (NCuts) \cite{Luxburg2007SC, Shi2000Ncuts}. The fundamental problem is how to build a good affinity matrix in these steps.

Inspired by recent advances in ${l_0}$-norm and ${l_1}$-norm techniques \cite{Tibshiran1996Lasso, Donoho2006MinimalL1Norm, Cand2008l1norm}, the introduction of sparse representation based techniques has resulted in enhanced separation ability in subspace clustering. \citet{Elhamifar2013SSC} proposed a sparse subspace clustering (SSC) algorithm to cluster data points lying in the union of low-dimensional subspaces. SSC considers that each data point can be represented as a sparse linear combination of other points by solving an ${l_1}$-norm minimization problem. The ${l_1}$-norm minimization program can be solved efficiently using convex programming tools. If the subspaces are either independent or disjoint under the appropriate conditions, SSC succeeds in recovering the desired sparse representations. After obtaining the desired sparse representation to define an affinity matrix, spectral clustering techniques are used to obtain the final clustering results. SSC shows very promising results in practice. \citet{Nasihatkon2011SSC} further analyzed connectivity within each subspace based on the connection between the sparse representations through ${l_1}$-norm minimization. \citet{Wang2013NSSC} extended SSC by adding either adversarial or random noise to study the behavior of sparse subspace clustering. However, some critical problems remain unsolved. In particular, the above techniques find the sparsest representation of each sample individually, which leads to high computational cost. Besides, a global structural constraint on the sparse representation is lacking, i.e., there is no theoretical guarantee that the nonzero coefficients correspond to points in the same subspace in the presence of corrupted data.

Low-rank representation based techniques have been proposed to address these drawbacks \cite{Liu2010LRR, Bao2012IRPCA, Vidala2013LRSC}. \citet{Liu2010LRR} proposed the low-rank representation (LRR) method to learn a low-rank representation of data by capturing the global structure of the data. The LRR method essentially requires singular value decomposition (SVD) at each iteration and needs hundreds of iterations before convergence. The computational complexity of LRR becomes computationally impracticable if the dimensionality of the samples is extremely large. Although an inexact variation of the augmented Lagrange multiplier (ALM) method \cite{Lin2010ALM, Lin2011LADM}, which is used to solve the optimization problem in LRR, performs well, and generally converges adequately in many practical applications, its convergence property still lacks a theoretical guarantee. \citet{Vidala2013LRSC} considered low-rank subspace clustering (LRSC) as a non-convex matrix decomposition problem, which can be solved in closed form using SVD of the noisy data matrix. Although LRSC can be carried out on data contaminated by noise with reduced computational cost, the clustering performance could be seriously degraded owing to the presence of such corrupted data. \citet{Chen2014SC} presented a low-rank representation with symmetric constraints (LRRSC) method. LRRSC further exploits the angular information of the principal directions of the symmetric low-rank representation for improved performance. However, LRRSC cannot avoid iterative SVD computations either, which is still time consuming. Consequently, LRRSC suffers from heavy computational cost when computing a symmetric low-rank representation. To obtain a good affinity matrix for spectral clustering using low-rank representation techniques, which can lead to higher performance and lower computational cost, low-rank representation of high-dimensional data still deserves investigation.

In this paper, we address the problem of subspace clustering by introducing the symmetric low-rank representation (SLRR) method. SLRR can be regarded as an improvement of our previous work, i.e., LRRSC \cite{Chen2014SC}. Figure \ref{fig:graphicalabstracts} shows an intuitive clustering example using five subjects to illustrate our approach. Owing to the self-expressiveness property of the data, our motivation starts from an observation of collaborative representation, which plays an important role in classification and clustering tasks \cite{Zhang2011SRCR, Lu2012LSR}. In particular, our motivation is to integrate the collaborative representation combined with low-rank matrix recovery techniques into a low-rank representation to learn a symmetric low-rank representation. The representation matrix involves the symmetric and low-rankness property of high-dimensional data representation, thereby preserving the low-dimensional subspace structures of high-dimensional data. An alternative low-rank matrix can be obtained by making use of the low-rank matrix recovery techniques closely related to the specific clustering problems. In contrast with ${l_1}$-norm minimization or iterative shrinkage, SLRR obtains a symmetric low-rank representation in a closed form solution by solving the symmetric low-rank optimization problem. Thereafter, an affinity graph matrix can be constructed by computing the angular information of the principal directions of the symmetric low-rank representation for spectral clustering. Further details are discussed in Section \ref{sec:SLRR}.

The proposed SLRR method has several advantages:

\begin{enumerate}[a)]

\item It incorporates collaborative representation combined with low-rank matrix recovery techniques into a low-rank representation, and can successfully learn a symmetric low-rank representation, which preserves the multiple subspace structure, for subspace clustering.
\item A symmetric low-rank representation can be obtained in a closed form solution by the symmetric low-rank optimization problem, which is similar to solve a regularized least squares regression. Consequently, it avoids iterative SVD operations, and can be employed by large-scale subspace clustering problems with the advantages of computational stability and efficiency.
\item Compared with state-of-the-art methods, our experimental results using benchmark databases demonstrate that the proposed method not only achieves competitive performance, but also dramatically reduces computational cost.
\end{enumerate}

The remainder of the paper is organized as follows. A brief overview of some existing work on rank minimization is given in Section \ref{sec:Relatedwork}. Section \ref{sec:SLRR} provides a detailed description of the proposed SLRR for subspace clustering. Section \ref{sec:Experiments} presents the experiments to evaluate the proposed SLRR on benchmark databases. Finally, Section \ref{sec:Conclusions} concludes the paper.

\section{Review of previous work}
\label{sec:Relatedwork}

Let $X = [{x_1},{x_2}...,{x_n}] \in {\mathbf{R}^{d \times n}}$ be a set of $d$-dimensional data vectors drawn from the union of $k$ subspaces $\{ {S_i}\} _{i = 1}^k$ of unknown dimensions. Without loss of generality, we can assume $X = [{X_1},{X_2},...,{X_k}]$, where ${X_i}$ consists of the vectors of ${S_i}$. The task of subspace clustering involves clustering data vectors into the underlying subspaces. This section provides a review of low-rank representation techniques for subspace clustering.

\citet{Liu2010LRR} proposed the LRR method for subspace clustering. In the absence of noise, LRR solves the following rank minimization problem:
\begin{equation}\label{eq:lrr1}
\mathop {\min }\limits_{Z,E} rank(Z) \qquad s.t. \qquad X = AZ,
\end{equation}
where $A = [{a_1},{a_2},...,{a_n}] \in {\mathbf{R}^{d \times n}}$ is an overcomplete dictionary. Since problem \eqref{eq:lrr1} is non-convex and NP-hard, LRR uses the nuclear norm as a common surrogate for the rank function:
\begin{equation}\label{eq:lrr2}
\mathop {\min }\limits_{Z} {\left\| Z \right\|_*} \qquad s.t.\qquad X = AZ,
\end{equation}
where ${\left\| \cdot \right\|_*}$ denotes the nuclear norm (that is, the sum of the singular values in the matrix).

In the case of data grossly corrupted by noise or outliers, LRR solves the following convex optimization problem:
\begin{equation}\label{eq:lrr2}
\mathop {\min }\limits_{Z,E} {\left\| Z \right\|_*} + \lambda {\left\| E \right\|_l}\qquad s.t.\qquad X = AZ + E,
\end{equation}
where $\lambda > 0$ is a parameter to balance the effects of the low-rank representation and errors, and \({\left\|  \cdot  \right\|_l}\) indicates a certain regularization strategy for characterizing various corruptions. For instance, the ${l_{2,1}}$-norm characterizes the error term that encourages the columns of error matrix $E$ to be zero. LRR uses the actual data $X$ as the dictionary. The above optimization problem can be efficiently solved by the inexact augmented Lagrange multipliers (ALM) method \cite{Lin2010ALM}. A post-processing step involves using $Z$ to construct the affinity matrix as $\left| Z \right| + {\left| Z \right|^T}$, which is symmetric and entrywise nonnegative. The final data clustering result is obtained by applying spectral clustering to the affinity matrix.

\section{Symmetric low-rank representation}
\label{sec:SLRR}

In this section, we discuss the core of the proposed method, which is to learn a SLRR for subspace clustering. The SLRR was inspired by collaborative representation and low-rank representation techniques, which are used in classification and subspace clustering \cite{Liu2010LRR, Zhang2011SRCR, Lu2012LSR}. This proposed technique identifies clusters using the angular information of the principal directions of the symmetric low-rank representation, which preserves the low-rank subspace structures. In particular, we first analyze the symmetric low-rank property of high-dimensional data representation based on the symmetric low-rank optimization problem, which is closely related with the regularized least squares regression. Then, we attempt to find an alternative low-rank matrix instead of the original data combined with low-rank matrix recovery techniques to obtain a symmetric low-rank representation. We further give the equivalence analysis of optimal solutions between problem \eqref{eq:SLRR} and \eqref{eq:TheFinalSLRR}. Finally, we construct the affinity graph matrix for spectral clustering, which completes the procedure for the SLRR method.

\subsection{The symmetric low-rank representation model}
\label{sec:SLRRRLSR}

In the absence of noise, i.e., the samples are strictly drawn from multiple subspaces, several criteria are imposed on the optimization models to learn the representation of samples as an affinity matrix for spectral clustering to solve the subspace clustering problem exactly \cite{Elhamifar2013SSC, Liu2010LRR, Wang2011ESS, Lu2012LSR}. For example, SSC employs the sparsest representation using an ${l_1}$-norm regularization, while LRR seeks to learn the lowest-rank representation using a nuclear-norm regularization. Both of these techniques can realize an affinity matrix of the samples involving the block diagonal between-clusters property, which reveals the membership of subspaces with theoretical guarantees. By considering noise and corruption in real observations, the lowest-rank criterion shows promising robustness among these criteria by capturing the global structure of the samples.

As mentioned above, the lowest-rank criterion, such as LRR, typically requires calculating singular value decomposition iteratively. This means that it becomes inapplicable both in terms of computational complexity and memory storage when the dimensionality of the samples is extremely large. To alleviate these problems, we design a new criterion as the convex surrogate of the nuclear norm. It is worth noting that we are intersected in a symmetric low-rank representation of a given data set. But we are not interested in seeking the best low-rank matrix recovery and completion using the obtained symmetric low-rank representation. The proposed method differs from the LRR method in terms of its matrix recovery. In the case of noisy and corrupted data, we seek to find a symmetric low-rank representation using collaborative representation combined low-rank representation techniques. The optimization problem is as follows:

\begin{equation}\label{eq:SLRR1}
\begin{split}
& \mathop {\min }\limits_Z rank(Z) + \alpha \left\| {X - XZ} \right\|_F^2 + \frac{\lambda }{2}trace({Z^T}Z)
\\
& s.t. \quad X = XZ + E, Z = {Z^T}.
\end{split}
\end{equation}
The discrete nature of the rank function makes it difficult to solve Problem \eqref{eq:SLRR1}. Many researchers have instead used the nuclear norm to relax the optimization problem \cite{Liu2010LRR, Ni2010LRRPSD, Ni2010LRRPSD, Chen2014SC}. Unfortunately, these methods cannot completely avoid the need for iterative singular value decomposition (SVD) operations, which incur a significant computational cost. Unlike these LRR-based methods that solve the nuclear norm problem, the following convex optimization provides a good surrogate for problem \eqref{eq:SLRR1}:

\begin{equation}\label{eq:SLRR}
\begin{split}
& \mathop {\min }\limits_Z \left\| {X - XZ} \right\|_F^2 + \frac{\lambda }{2}trace({Z^T}Z)
\\
& s.t. \quad X = XZ + E, Z = {Z^T}, rank(Z) \le r,
\end{split}
\end{equation}
where ${\left\|  \cdot  \right\|_F}$ denotes the Frobenius norm of the matrix, $\lambda  > 0$ is a parameter used to balance the effects of the two parts, and $r \in {\rm N}$ is a parameter used to guarantee the low-rank representation $Z$. To maintain the weight consistency of each pair of data points, we impose a symmetric constraint on representation $Z$. Then, by imposing a low-rank constraint on representation $Z$, we obtain a desired symmetric low-rank representation.

To further analyze problem \eqref{eq:SLRR}, we first simplify this optimization problem. Removing the constraint $rank(Z) \le r$ from the problem leads to another optimization problem:

\begin{equation}\label{eq:NewSLRR}
\mathop {\min }\limits_Z \left\| {X - XZ} \right\|_F^2 + \frac{\lambda }{2}trace({Z^T}Z) \quad s.t. \quad X = XZ + E, Z = {Z^T}.
\end{equation}

The solution of SLRR in problem \eqref{eq:NewSLRR} can be analytically obtained

\begin{equation}
{Z^{\ast}} = {({X^T}X + \lambda \cdot I)^{ - 1}}{X^T}X
\end{equation}
where $I$ is the identity matrix, and $\lambda >0$ is a parameter.

Next, we show that $Z^{\ast}$ is symmetric.

\begin{theorem}
\label{theorem1}
The matrix
\[Z = {({X^T}X + \lambda  \cdot I)^{ - 1}}{X^T}X\]
is symmetric.
\end{theorem}

\begin{proof}
Clearly,
\begin{equation*}
\begin{split}
Z & = {({X^T}X + \lambda  \cdot I)^{ - 1}}{X^T}X
\\
& = {({X^T}X + \lambda  \cdot I)^{ - 1}} \cdot ({X^T}X + \lambda  \cdot I - \lambda  \cdot I)
\\
& = I - \lambda  \cdot {({X^T}X + \lambda  \cdot I)^{ - 1}}.
\end{split}
\end{equation*}

On the other hand,
\begin{equation*}
\begin{split}
{Z^T} & = {X^T}X{({X^T}X + \lambda  \cdot I)^{ - 1}}
\\
& = ({X^T}X + \lambda  \cdot  I - \lambda  \cdot I) \cdot {({X^T}X + \lambda  \cdot I)^{ - 1}}
\\
& = I - \lambda  \cdot {({X^T}X + \lambda  \cdot I)^{ - 1}}.
\end{split}
\end{equation*}
Since $Z={Z^T}$, $Z$ is symmetric.
\end{proof}

It should be pointed out that ${Z^{\ast}}$ may not be a low-rank matrix. Denote the ranks of ${Z^{\ast}}$ and $X$ by $rank({Z^{\ast}})$ and $rank(X)$, respectively. It is easy to see that $rank({Z^{\ast}}) \le rank(X)$. As the noises we confront are ubiquitous in practice, $X$ is not a low-rank matrix. This implies that the real data may not strictly follow subspace structures because of noise or corruption.

In general, ${Z^{\ast}}$ is a low-rank matrix if $X$ is low-rank. If we require that ${\rm{rank}}(X) \le r$, where $r$ is some small positive integer, then ${Z^{\ast}}$ is a symmetric low-rank matrix. If we can use an alternative low-rank matrix $A$ to replace $X$, a desired low-rank solution could be obtained. We propose the following convex optimization provides a good surrogate for the problem \eqref{eq:SLRR}.
\begin{equation}\label{eq:TheFinalSLRR}
\begin{split}
& \mathop {\min }\limits_Z \left\| {A - AZ} \right\|_F^2 + \frac{\lambda }{2}trace({Z^T}Z)
\\
& s.t. \quad A = AZ + E, Z = {Z^T}, rank(Z) \le r.
\end{split}
\end{equation}
If $rank(A) \le r$, then
\begin{equation}
{Z^{\ast}} = {({A^T}A + \lambda \cdot I)^{ - 1}}{A^T}A
\end{equation}
is the analytical optimal solution to problem \eqref{eq:TheFinalSLRR}.

From linear algebra, ${Z^{\ast}}$ is a symmetric low-rank matrix if $A$ is low-rank. The only remaining issue is how to get an alternative low-rank matrix $A$ instead of $X$ from a given set of data.

\subsection{Pursuing an alternative low-rank matrix through low-rank matrix recovery techniques}
\label{sec:ALRM}

Given the assumption mentioned above, data points are approximately drawn from a union of subspaces. Each data point can be represented by a linear combination of the other data points. Therefore, it is reasonable that a low-rank matrix recovered from corrupt observations is employed instead of the original data in problem \eqref{eq:NewSLRR}. Here we consider in detail, three implementations of an alternative low-rank matrix from corrupt observations.

First, we explain the idea behind the first implementation. We incorporate low-rank matrix recovery techniques, the choice of which is closely related to the specific problem, into recovering corrupt samples. For example, it is well known that principal component analysis (PCA) is one of the most popular dimension reduction techniques for face images \cite{Jolliffe2002PCA}. PCA assumes that the data is drawn from a single low-dimensional subspace. In fact, our experiments demonstrate its effectiveness when applied to face clustering and motion segmentation. In particular, PCA learns a low-rank project matrix $P \in {\mathbf{R}^{m \times r}}$ by minimizing the following problem:
\begin{equation}
\mathop {\min }\limits_P \left\| {X - P{P^T}X} \right\|_F^2 \quad s.t. \quad {P^T}P = {I_r}.
\end{equation}
Let $A = P{P^T}X$. Note that $A$ is low-rank matrix recovery of $X$. If $rank(P) \le r$, a globally optimal solution ${Z^{\ast}} = {({A^T}A + \lambda \cdot I)^{ - 1}}{A^T}A$ of problem \eqref{eq:TheFinalSLRR} can be obtained in closed form. Obviously, it is a symmetric low-rank matrix. The low-rank matrix recovery reveals its vital importance in learning a low-rank representation.

It is well known that PCA is an effective method when the data are corrupted by Gaussian noise. However, its performance is limited in real applications by a lack of robustness to gross errors. The second implication for consideration is to recover a low-rank matrix from highly corrupted observations. For example, RPCA decomposes the data matrix $X$ into the sum of a low-rank approximation $A$ and an additive error $E$ \cite{Wright2009RPCA, Candes2011RobustPCA}, which leads to the following convex problem:
\begin{equation}
\mathop {\min }\limits_A {\left\| A \right\|_*} + \lambda {\left\| E \right\|_1} \quad s.t. \quad X = A + E.
\end{equation}
Assume that the optimal solution to this problem is ${A^{\ast}}$, where ${A^{\ast}}$, where ${A^{\ast}}$ is a low-rank matrix. If $rank({A^{\ast}}) \le r$, a globally optimal solution ${Z^{\ast}} = {({A^{\ast}}^T{A^{\ast}}+ \lambda \cdot I)^{ - 1}}{A^{\ast}}^T{A^{\ast}}$  can be obtained for problem \eqref{eq:TheFinalSLRR}.

Besides, we further consider incorporating feature extraction into the low-rank representation. We use low-rank features extracted from the corrupted samples instead of the original data by dimension reduction techniques. We also use the face clustering example to illustrate the importance and feasibility of feature extraction. Random features can be viewed as a less-structured face feature. Randomfaces are independent of the face images \cite{Kaski1998RM, Bingham2001RP}. A low-rank transform matrix $P \in {\mathbf{R}^{m \times r}}$, whose entries are independently sampled from a zero-mean normal distribution, is extremely efficient to generate, whose entries are independently sampled from zero-mean normal distribution. The random project (RP) matrix $P$ can be used to for dimension reduction for of face images. Let $A = {P^T}X$, where $A$ is an extracted feature matrix. A globally optimal solution ${Z^{\ast}} = {({A^T}A + \lambda \cdot I)^{ - 1}}{A^T}A$  to problem \eqref{eq:TheFinalSLRR} can also be obtained.

To examine the connection among the low-rank matrix recovery techniques reliant on dimension reduction, we consider the special case in which the low-rank projection matrix $P$ has orthogonal columns, i.e., ${P^T}P = I$. Assuming that ${P^T}P = I$, both of the implications, i.e., $A = P{P^T}X$ and $A = {P^T}X$, are equivalent to each other in problem \eqref{eq:TheFinalSLRR}. This is summarized by the following Lemma.

\begin{lemma}
\label{lemma1}
Let $D$, $U$, and $V$ be matrices of compatible dimensions. Suppose $U$ and $V$ have orthogonal columns, i.e., ${U^T}U = I$ and ${V^T}V = I$, then we have
\[{\left\| D \right\|_F} = {\left\| {UD{V^T}} \right\|_F}.\]
\end{lemma}

\begin{proof}

By definition of the Frobenius norm, we have
\begin{equation*}
\begin{split}
 {\left\| {UD{V^T}} \right\|_F} & = \sqrt {trace({{(UD{V^T})}^T}(UD{V^T}))}
\\
& = \sqrt {trace(V{D^T}{U^T}UD{V^T})}.
\end{split}
\end{equation*}

As ${U^T}U = I$ and ${V^T}V = I$, we have
\begin{equation*}
\begin{split}
 {\left\| {UD{V^T}} \right\|_F} & = \sqrt {trace(V{D^T}D{V^T})}
\\
& = \sqrt {trace({V^T}V{D^T}D)}  = {\left\| D \right\|_F}.
\end{split}
\end{equation*}

\end{proof}

According to Lemma \ref{lemma1}, we can conclude that ${\left\| X \right\|_F} = {\left\| {PX} \right\|_F}$, where the low-rank project matrix $P$ has orthogonal columns. Consequently, $A = P{P^T}X$ or $A = {P^T}X$ are alternatives to obtain the same globally optimal solution of problem \eqref{eq:TheFinalSLRR}. The computational cost of the first implementation can be effectively reduced by using a simplified version if the low-rank project matrix has orthogonal columns.

The use of low-rank matrix recovery techniques to improve the performance of many applications is not in itself surprising. However, in this paper, the main purpose of using such techniques is to derive an alternative low-rank matrix that can be used to obtain the symmetric low-rank representation discussed above.

\subsection{Equivalence analysis of optimal solutions}
\label{sec:Equivalence}

In Section \ref{sec:SLRRRLSR}, we first introduced problem \eqref{eq:SLRR} to describe a symmetric low-rank representation model, and then considered this problem as the surrogate of an alternative low-rank matrix. We then analyzed the equivalence between problems (5) and (8) in terms of the optimal solution.

Let us first consider a specific case of low-rank matrix recovery techniques, such as PCA . In PCA, the low-rank projection matrix $P$ is an orthogonal matrix, i.e., ${P^T}P = I$. Then, problem \eqref{eq:SLRR} can be converted into an equivalent problem \eqref{eq:TheFinalSLRR} according to Lemma \ref{lemma1}. Consequently, the globally optimal solution of problem  \eqref{eq:TheFinalSLRR},  ${Z^{\ast}} = {({A^T}A + \lambda \cdot I)^{ - 1}}{A^T}A$, is the same as that of problem \eqref{eq:SLRR}. It is clear that ${Z^{\ast}}$ is a symmetric low-rank representation that preserves the multiple subspace structure.

Furthermore, we note the remaining cases of low-rank matrix recovery techniques, such as RP and RPCA . For example, the columns of the low-rank projection matrix   may not be orthogonal to one another, or an alternative low-rank matrix recovered from the original data may not be directly obtained using the low-rank projection matrix. Thus, we cannot calculate the globally optimal solution of problem \eqref{eq:SLRR} directly since its solution is intractable. To address this problem, we integrate an alternative low-rank matrix into problem \eqref{eq:TheFinalSLRR} to learn a symmetric low-rank representation. As mentioned above, problem \eqref{eq:TheFinalSLRR} can be solved as a closed form solution. It should be emphasized that this surrogate is reasonable for the following two reasons: (1) high-dimensional data often lie close to low-dimensional structures; and (2) the alternative matrix recovered from the original data has low rank. Such a symmetric low-rank representation can also preserve the multiple subspace structure.

\subsection{Construction of an affinity graph matrix for subspace clustering}
\label{sec:SM}

Using the symmetric low-rank matrix ${Z^{\ast}}$ from problem \eqref{eq:TheFinalSLRR}, we need to construct an affinity graph matrix $W$. We consider ${Z^{\ast}}$ with the skinny SVD ${U^{\ast}}{\sum ^{\ast}}{({V^{\ast}})^T}$, and define $M = {U^{\ast}}{({\sum ^{\ast}})^{{1 \mathord{\left/
 {\vphantom {1 2}} \right.
 \kern-\nulldelimiterspace} 2}}},N = {({\sum ^{\ast}})^{{1 \mathord{\left/
 {\vphantom {1 2}} \right.
 \kern-\nulldelimiterspace} 2}}}{({V^{\ast}})^T}$. As suggested in \cite{Chen2014SC}, we apply the mechanism of driving the construction of the affinity graph from matrix ${Z^{\ast}}$. This considers the angular information from all row vectors of matrix $M$ or all column vectors of matrix $N$ to define an affinity graph matrix as follows:
 \begin{equation}
{[W]_{ij}} = {\left( {\frac{{m_i^T{m_j}}}{{{{\left\| {{m_i}} \right\|}_2}{{\left\| {{m_j}} \right\|}_2}}}} \right)^{2\alpha }}
\quad or \quad
{[W]_{ij}} = {\left( {\frac{{n_i^T{n_j}}}{{{{\left\| {{n_i}} \right\|}_2}{{\left\| {{n_j}} \right\|}_2}}}} \right)^{2\alpha }},
\end{equation}
where ${m_i}$ and ${m_j}$ represent the $i$-th and $j$-th rows of matrix $M$, and ${n_i}$ and ${n_j}$ represent the $i$-th and $j$-th columns of matrix $N$, respectively, and $\alpha  \in {\rm N}$ is a parameter to adjust the sharpness of the affinity between different clusters. Algorithm \ref{alg:SLRR} summarizes the complete subspace clustering algorithm for SLRR.

\begin{algorithm}[!htbp]
\renewcommand{\algorithmicrequire}{\textbf{Input:}}
\renewcommand\algorithmicensure {\textbf{Output:} }
\caption{The SLRR algorithm}
\label{alg:SLRR}
\begin{algorithmic}[1]
\REQUIRE ~~\\
data matrix $X = [{x_1},{x_2},...,{x_n}] \in {\mathbf{R}^{m \times n}}$, number of subspaces $k$, regularized parameters $\lambda > 0, \alpha \in {\rm N}, r \in {\rm N}$

\STATE Recover an alternative low-rank matrix $A$ from $X$ using low-rank matrix recovery techniques such as RPCA. Alternatively, learn the low-rank projection $P$ from $X$ using low-rank matrix recovery techniques, and then obtain an alternative low-rank matrix $A = {P^T}PX$, or an alternative feature matrix $A = PX$.

\STATE Solve the following problem:

\begin{equation*}
\begin{split}
& \mathop {\min }\limits_Z \left\| {A - AZ} \right\|_F^2 + \frac{\lambda }{2}trace({Z^T}Z)
\\
& s.t. \quad A = AZ + E, Z = {Z^T}, rank(Z) \le r.
\end{split}
\end{equation*}

and obtain the optimal solution ${Z^{\ast}} = {({A^T}A + \lambda \cdot I)^{ - 1}}{A^T}A$.

\STATE Compute the skinny SVD ${Z^{\ast}}={U^{\ast}}{\sum ^{\ast}}{({V^{\ast}})^T}$.
\STATE Calculate $M = {U^{\ast}}{({\sum ^{\ast}})^{{1 \mathord{\left/ {\vphantom {1 2}} \right. \kern-\nulldelimiterspace} 2}}}$ or $N = {({\sum ^{\ast}})^{{1 \mathord{\left/ {\vphantom {1 2}} \right. \kern-\nulldelimiterspace} 2}}}{({V^{\ast}})^T}$.
\STATE Construct the affinity graph matrix $W$, i.e.,
\begin{equation*}
{[W]_{ij}} = {\left( {\frac{{m_i^T{m_j}}}{{{{\left\| {{m_i}} \right\|}_2}{{\left\| {{m_j}} \right\|}_2}}}} \right)^{2\alpha }}
\quad or \quad
{[W]_{ij}} = {\left( {\frac{{n_i^T{n_j}}}{{{{\left\| {{n_i}} \right\|}_2}{{\left\| {{n_j}} \right\|}_2}}}} \right)^{2\alpha }}.
\end{equation*}
\STATE Apply $W$ to perform NCuts.

\ENSURE ~~\\ The clustering results.
\end{algorithmic}
\end{algorithm}

Assume that the size of $X$ is $m \times n$, where $X$ has $n$ samples and each sample has $m$ dimensions. For convenience, we apply PCA as an example low-rank matrix recovery technique to illustrate the computational complexity of Algorithm \ref{alg:SLRR}. Thus, the computational complexity of the first two steps in Algorithm \ref{alg:SLRR} is $O({m^2}n + {m^3})$, while the computational complexity of the last four steps in Algorithm \ref{alg:SLRR} is $O(m{n^2} + {n^3})$. The complexity of Algorithm \ref{alg:SLRR} is $O({{m^2}n + m{n^2} + {m^3} + n^3})$. If $m \ll n$, the overall complexity of Algorithm \ref{alg:SLRR} is $O({n^3})$.

\section{Experiments}
\label{sec:Experiments}

\subsection{Experimental settings}

\begin{figure*}[!htbp]
\centering
\subfigure[The original sample images]{
\label{fig:yaleba}
\includegraphics[width=1\textwidth]{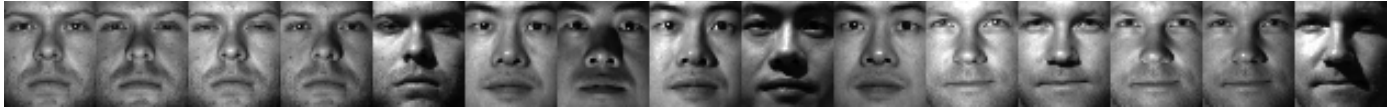}}
\subfigure[The corrupted sample images with the 20\% random pixel corruptions]{
\label{fig:yalebb}
\includegraphics[width=1\textwidth]{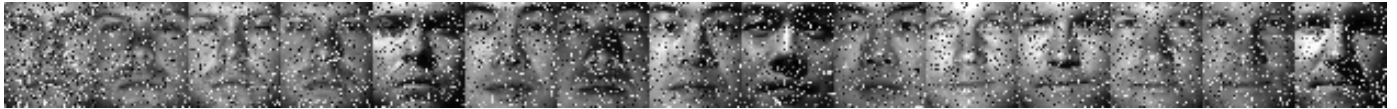}}
\subfigure[The corrected sample images by applying RPCA]{
\label{fig:yalebc}
\includegraphics[width=1\textwidth]{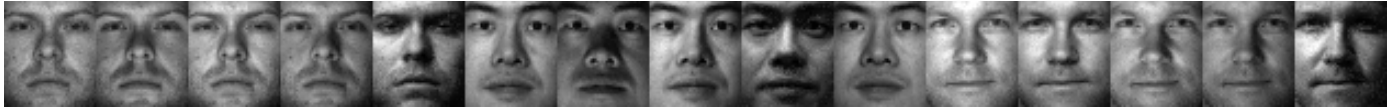}}
\label{fig:yaleb}
\caption{Example images of multiple individuals from the Extended Yale B database.}
\end{figure*}

\subsubsection{Databases}

To evaluate the SLRR, we performed different experiments on two popular benchmark databases, e.g., the extended Yale B and Hopkins 155 databases. The statistics of the two databases are summarized below.

\begin{enumerate}[$\bullet$]
\item Extended Yale B database \cite{Lee05YALEB, GeBeKr01YALEB}. This database contains 2414 frontal images of 38 individuals, with images of each individual lying in a low-dimensional subspace. There are around $59-64$ images available for each individual. To reduce the computational time and memory requirements of algorithms, we used a normalized face image with size $48 \times 42$ pixels in the experiments. Figure \ref{fig:yaleba} shows some example face images from the Extended Yale B Database.
\item Hopkins 155 database \cite{Hopkins155}. This database consists of 156 video sequences of two or three motions. Each video sequence motion corresponds to a low-dimensional subspace. There are $39-550$ data vectors drawn from two or three motions for each video sequence. Figure \ref{fig:h155example} shows some example frames from four video sequences with traced feature points.
\end{enumerate}

\subsubsection{Baselines and evaluation}

To investigate the efficiency and robustness of the proposed method, we compared the performance of SLRR with several state-of-the-art subspace clustering algorithms, such as LRR \cite{Liu2010LRR}, LRRSC \cite{Chen2014SC}, SSC \cite{Elhamifar2013SSC}, local subspace affinity (LSA) \cite{Yang2006MotionSeg}, and low rank subspace clustering (LRSC) \cite{Vidala2013LRSC}. For the state-of-the-art algorithms, we used the source code provided by the respective authors. The Matlab source code for our method is available online at http://www.machineilab.org/users/chenjie.

The subspace clustering error is the percentage of misclassified samples over all samples, which is measured as
\begin{equation}
error = \frac{{{N_{error}}}}{{{N_{total}}}},
\end{equation}
where ${{N_{error}}}$ denotes the number of misclassified samples, and ${{N_{total}}}$ is the total number of samples. For LRR, we reported the results after post-processing of the affinity graph. Moreover, we chose the noisy data version of $({P_3})$ of LRSC to show its results. All experiments were implemented on Matlab R2011b and performed on a personal computer with an Intel Core i5-2300 CPU and 16 GB memory.

\subsubsection{Parameter settings}

\begin{table}[!htbp]
\small
\setlength{\abovecaptionskip}{0pt}
\setlength{\belowcaptionskip}{10pt}
\setlength{\tabcolsep}{1pt}
\centering
\caption{Parameter settings for different algorithms on face clustering. For SLRR, $n$ is the number of subspaces, i.e., the number of subjects. Let SLRR${_{PCA}}$, SLRR${_{RPCA}}$ and SLRR${_{RP}}$ denote the application of PCA, RPCA, and RP for low-rank matrix recovery or dimension reduction in SLRR.}

\label{parmaterofface}
\begin{tabular}{|c|c|c|}
\hline
\multirow{2}{*}{Method} & \multicolumn{2}{c|}{Face clustering} \\
\cline{2-3}
 & Scenario 1 & Scenario 2 \\
\hline
  SLRR${_{PCA}}$ & $\alpha = 3, \lambda = 30, r = 50n$ & $\alpha = 2, \lambda = 40, r = 10n$  \\
\hline
  SLRR${_{RP}}$ & $\alpha = 3, \lambda = 1.2, r = 10n $ & $\alpha = 3, \lambda = 1, r = 10n $ \\
\hline
  SLRR${_{RPCA}}$ & $\alpha = 2, \lambda = 3, \lambda{_{RPCA}} \in  [0.02, 0.03] $ & - \\
\hline
 LRRSC & $ \lambda = 0.2, \alpha = 4$ & $ \lambda = 0.1,\alpha = 3$ \\
\hline
  LRR & \multicolumn{2}{c|}{$ \lambda = 0.18 $} \\
\hline
  SSC & ${\lambda _e} = {{8} \mathord{\left/
 {\vphantom {{800} {{\mu _e}}}} \right.
 \kern-\nulldelimiterspace} {{\mu _e}}}$ & ${\lambda _e} = {{20} \mathord{\left/
 {\vphantom {{800} {{\mu _e}}}} \right.
 \kern-\nulldelimiterspace} {{\mu _e}}}$ \\
\hline
  LSA & \multicolumn{2}{c|}{$K=3, d=5$} \\
\hline
  LRSC & $\tau  = 0.4,\alpha  = 0.045$ &  $\tau  = 0.045,\alpha  = 0.045$  \\
\hline
\end{tabular}
\end{table}

\begin{table}[!htbp]
\small
\setlength{\abovecaptionskip}{0pt}
\setlength{\belowcaptionskip}{10pt}
\centering
\caption{Parameter settings for different algorithms on motion segmentation. For SLRR, $n$ represents the number of motions in each video sequence.}
\label{parmaterofmotion}
\begin{tabular}{|c|c|c|c|c|}
\hline
\multirow{2}{*}{Method} & \multicolumn{2}{c|}{Motion segmentation} \\
\cline{2-3}
 & Scenario 1 & Scenario 2 \\
\hline
  SLRR & $\alpha = 2, \lambda = 5{e^{ - 3}}$ & {$ \alpha = 2, \lambda = 5{e^{ - 3}}, r=4n $} \\
\hline
 LRRSC & $\lambda = 3.3, \alpha = 2$ & $\lambda = 3, \alpha = 3$\\
\hline
  LRR & \multicolumn{2}{c|}{$ \lambda = 4 $} \\
\hline
  SSC & \multicolumn{2}{c|}{${\lambda _z} = {{800} \mathord{\left/
 {\vphantom {{800} {{\mu _z}}}} \right.
 \kern-\nulldelimiterspace} {{\mu _z}}}$} \\
\hline
  LSA & $K=8, d=5$   & $K=8, d=4$  \\
\hline
  LRSC & \multicolumn{2}{c|}{$\tau  = 420,\alpha  = 3000\,or\,\alpha  = 5000$ } \\
\hline
\end{tabular}
\end{table}

To obtain the best results of the state-of-the-art algorithms in the experiments, we either applied the optimal parameters for each method as given by the respective authors, or manually tuned the parameters of each method. We emphasize that SLRR is follow-up research based on our previous work, i.e., LRRSC \cite{Chen2014SC}. Hence, we reported the parameter settings and results of several algorithms from \cite{Chen2014SC}, e.g., LRR, LRRSC, SSC, LSA and LRSC, for comparison with the results of SLRR in our experiments. The parameters for these methods are set as shown in Tables \ref{parmaterofface} and \ref{parmaterofmotion}.

According to problem \eqref{eq:TheFinalSLRR}, SLRR has three parameters: $\lambda$, $\alpha$ and $r$. Empirically speaking, parameter $\lambda$ should be relatively large if the data are slightly contaminated by noise, and vice versa. In other words, parameter $\lambda$ is usually dependent on the prior of the error level of data. In fact, parameter $\lambda$ has a wide range in our experiments. Parameter $\alpha$ ranges from 2 to 4. To pursue an alternative low-rank matrix, parameter $r$ may be closely related with the intrinsic dimension of high-dimensional data. For example, images of an individual with a fixed pose and varying illumination lie close to a 9-dimensional linear subspace under the Lambertian assumption \cite{Basri2003}. Besides, the tracked feature point trajectories from a single motion lie in a linear subspace of at most four dimensions \cite{Boult1991Factor}. Therefore, we used $r=10n$ in some face clustering experiments shown in Table \ref{parmaterofface}, and $r=4n$ for the motion segmentation experiments shown in Table \ref{parmaterofmotion}, where $n$ denotes the number of subspaces. Note that we used $r=50n$ in SLRR${_{PCA}}$ for the first scenario of face clustering. However, using $r=10n$ for SLRR${_{PCA}}$ also achieves satisfactory performance in the face clustering experiments. Further results and discussions of the parameters are given in the respective sections for the experiments.

\subsection{Experiments on face clustering}

We first evaluated the clustering performance of SLRR as well as the other methods on the Extended Yale B database. Face clustering refers to the problem of clustering face images from multiple individuals according to each individual. The face images of the individual, captured under various laboratory-controlled lighting conditions, can be well approximated by a low-dimensional subspace \cite{Basri2003}. Therefore, the problem of clustering face images reduces to clustering a collection of images according to multiple subspaces. We considered two different clustering scenarios of face images to evaluate the performance of the proposed SLRR.

\subsubsection{First scenario for face clustering}

Following the experimental settings in \cite{Liu2010LRR1}, we chose a subset of the Extended Yale B database consisting of the 640 frontal face images from the first 10 subjects. We used two different low-rank matrix recovery techniques (PCA, RPCA) and one dimension reduction technique (RP) to implement SLRR for face clustering.

\begin{figure*}[!htbp]
\begin{minipage}[t]{0.24\linewidth}
\centering
\subfigure[$r=50$]{
\label{fig:facebypca:a} 
\includegraphics[width=1\textwidth]{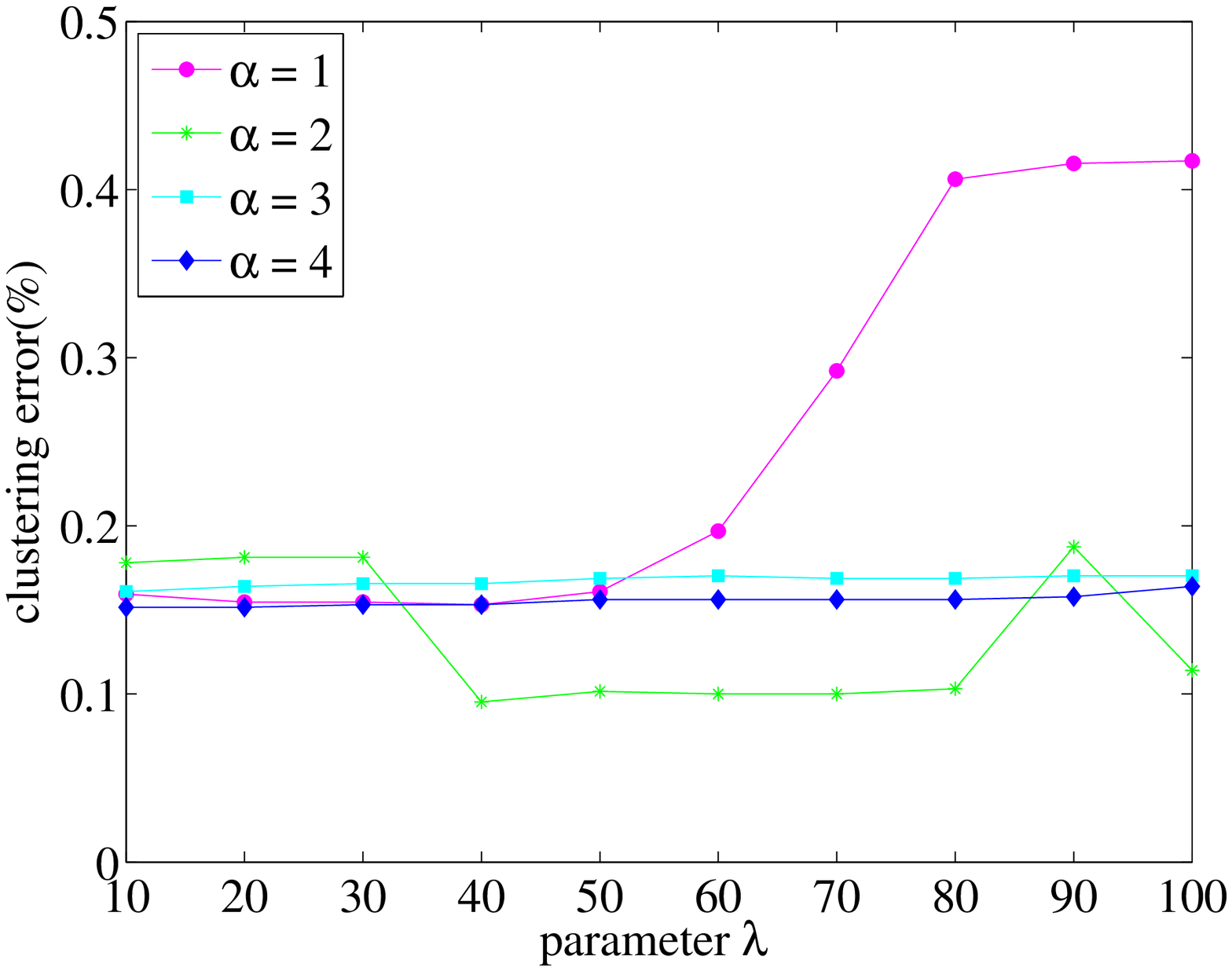}}
\end{minipage}%
\begin{minipage}[t]{0.24\linewidth}
\centering
\subfigure[$r=100$]{
\label{fig:facebypca:b} 
\includegraphics[width=1\textwidth]{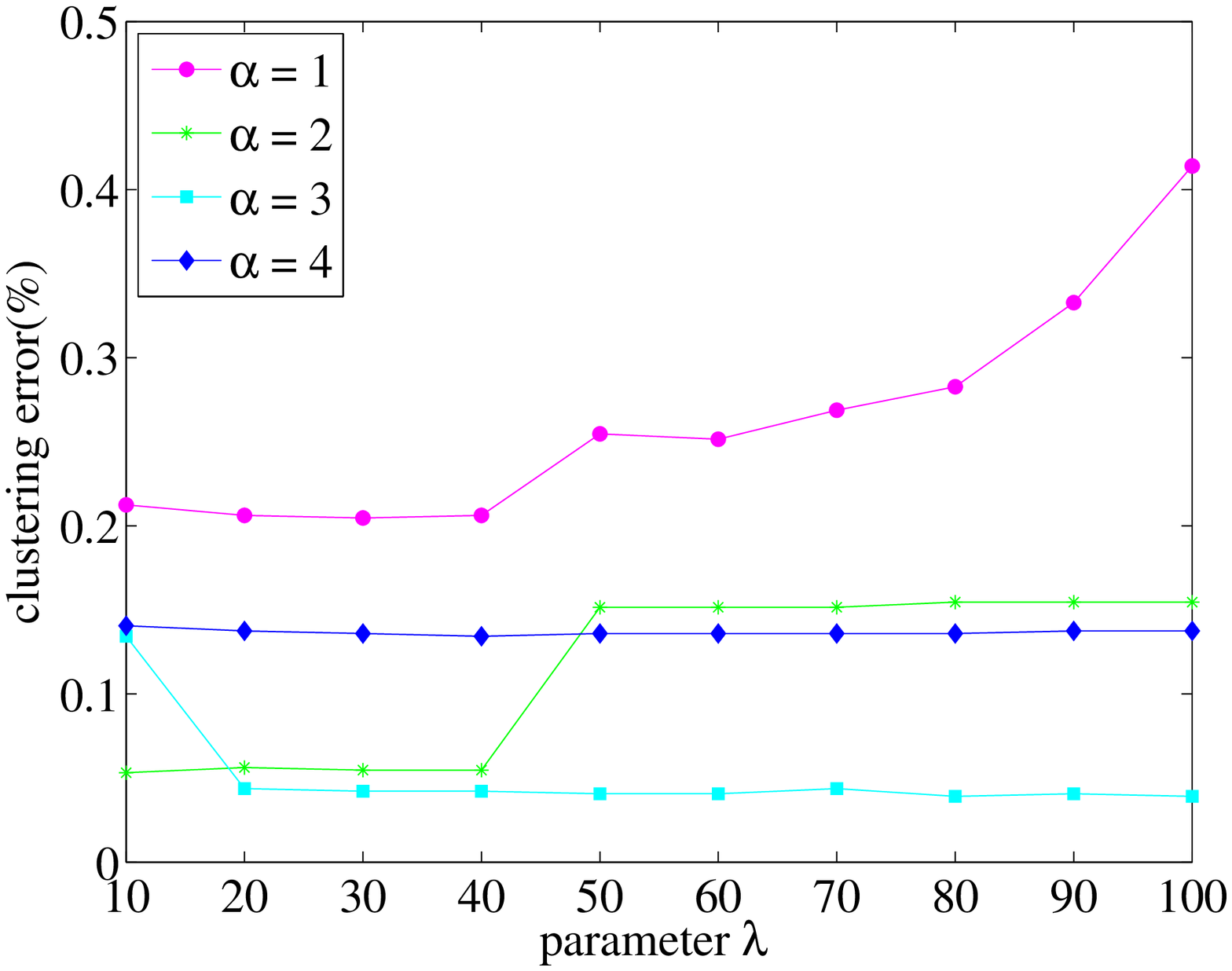}}
\end{minipage}%
\begin{minipage}[t]{0.24\linewidth}
\centering
\subfigure[$r=200$]{
\label{fig:facebypca:c} 
\includegraphics[width=1\textwidth]{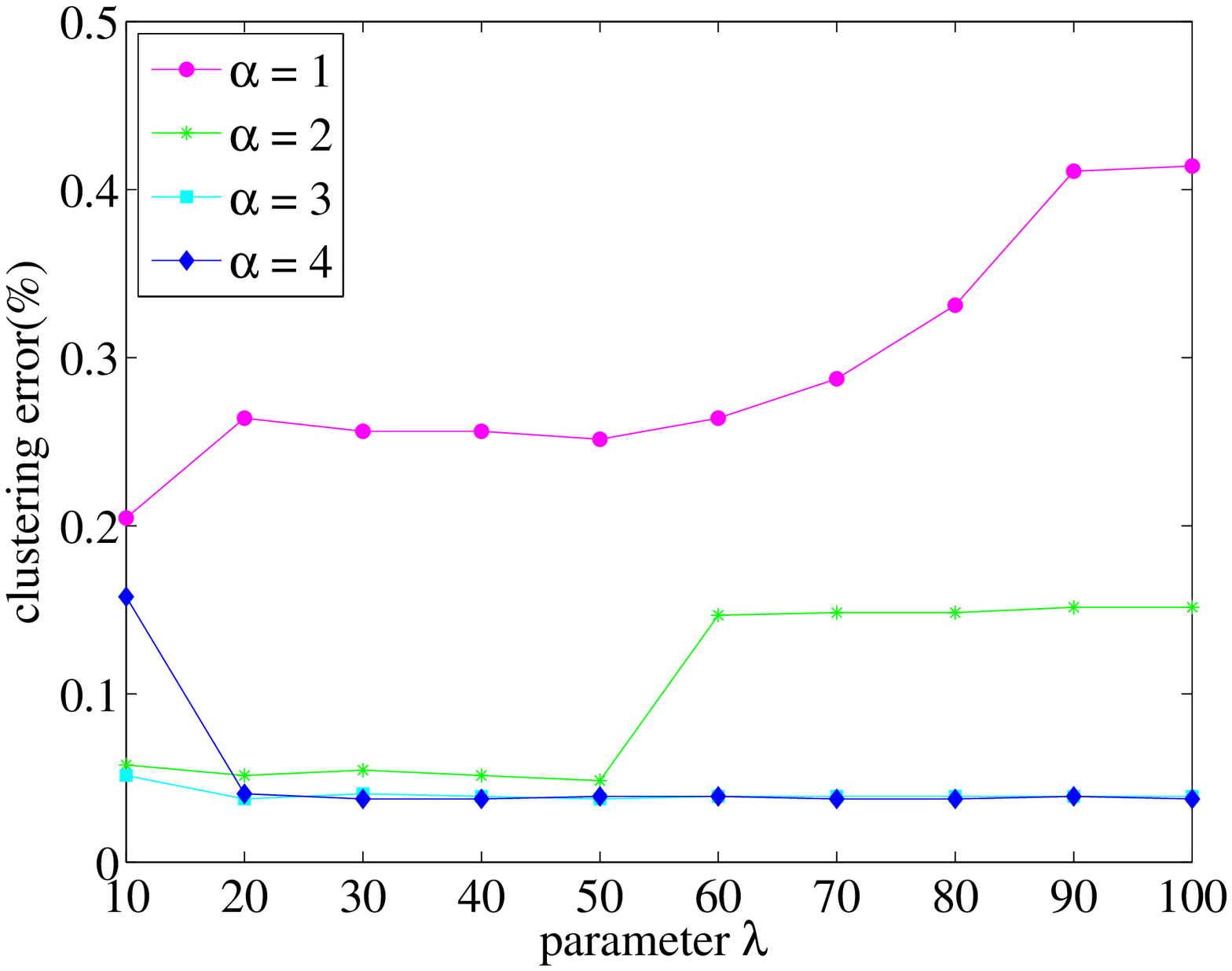}}
\end{minipage}%
\begin{minipage}[t]{0.24\linewidth}
\centering
\subfigure[r=500]{
\label{fig:facebypca:d} 
\includegraphics[width=1\textwidth]{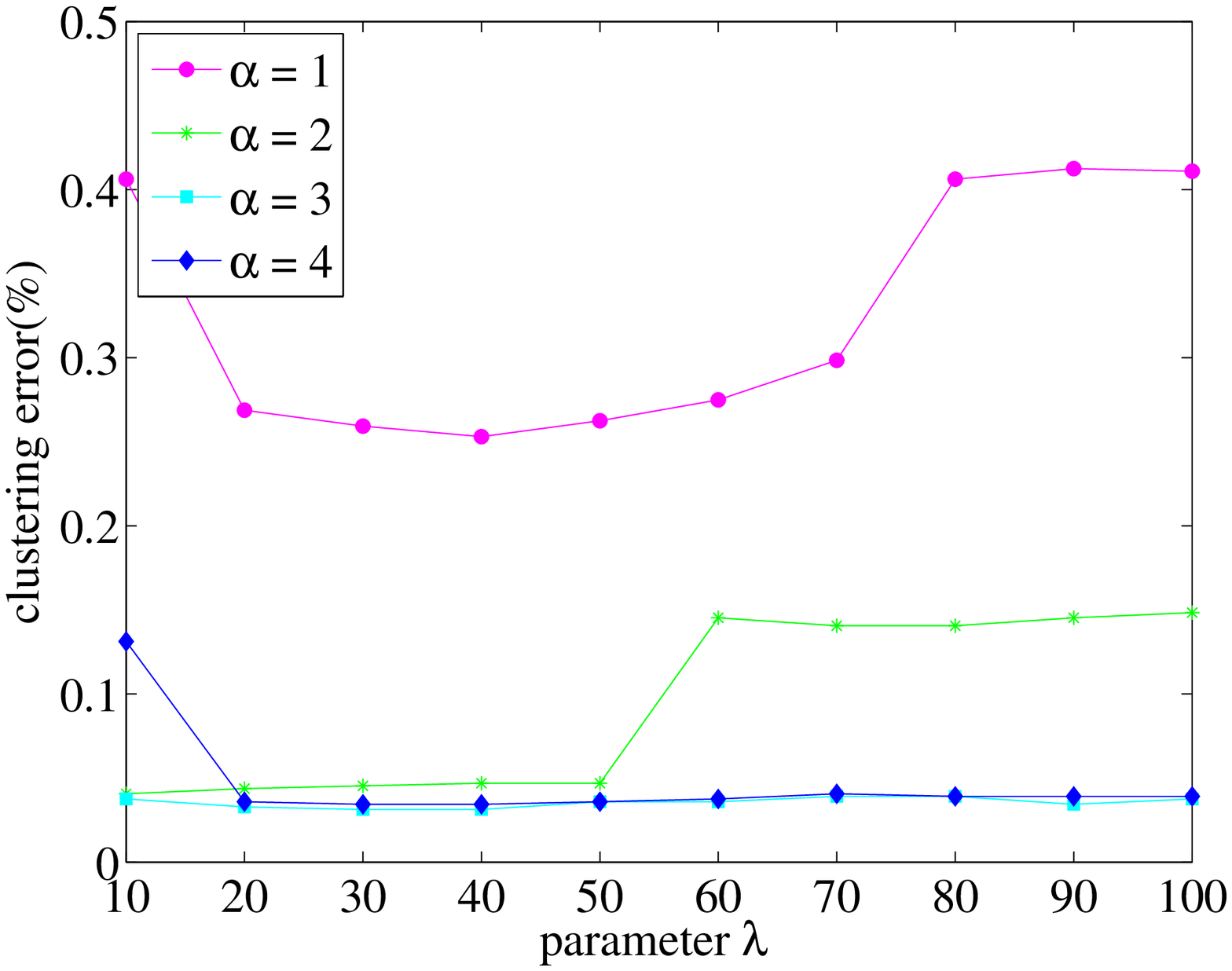}}
\end{minipage}%
\caption{Changes in clustering error when varying $\lambda$ and $\alpha$ under different $r$, by applying PCA on face images of the first 10 classes in the Extended Yale Database B.}
\label{fig:facebypca} 
\end{figure*}

\begin{figure*}[!htbp]
\begin{minipage}[t]{0.24\linewidth}
\centering
\subfigure[$r=50$]{
\label{fig:facebyrp:a} 
\includegraphics[width=1\textwidth]{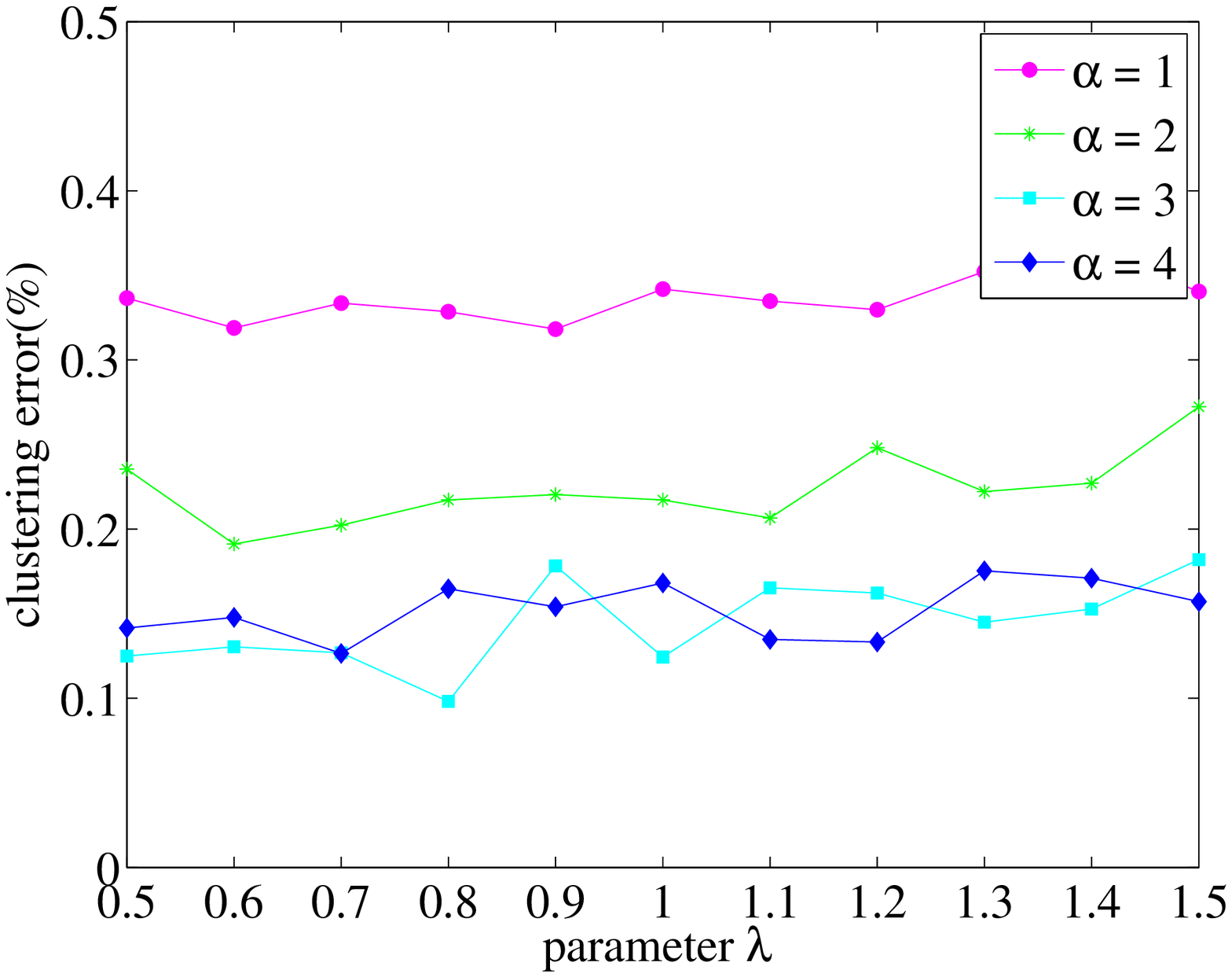}}
\end{minipage}%
\begin{minipage}[t]{0.24\linewidth}
\centering
\subfigure[$r=90$]{
\label{fig:facebyrp:b} 
\includegraphics[width=1\textwidth]{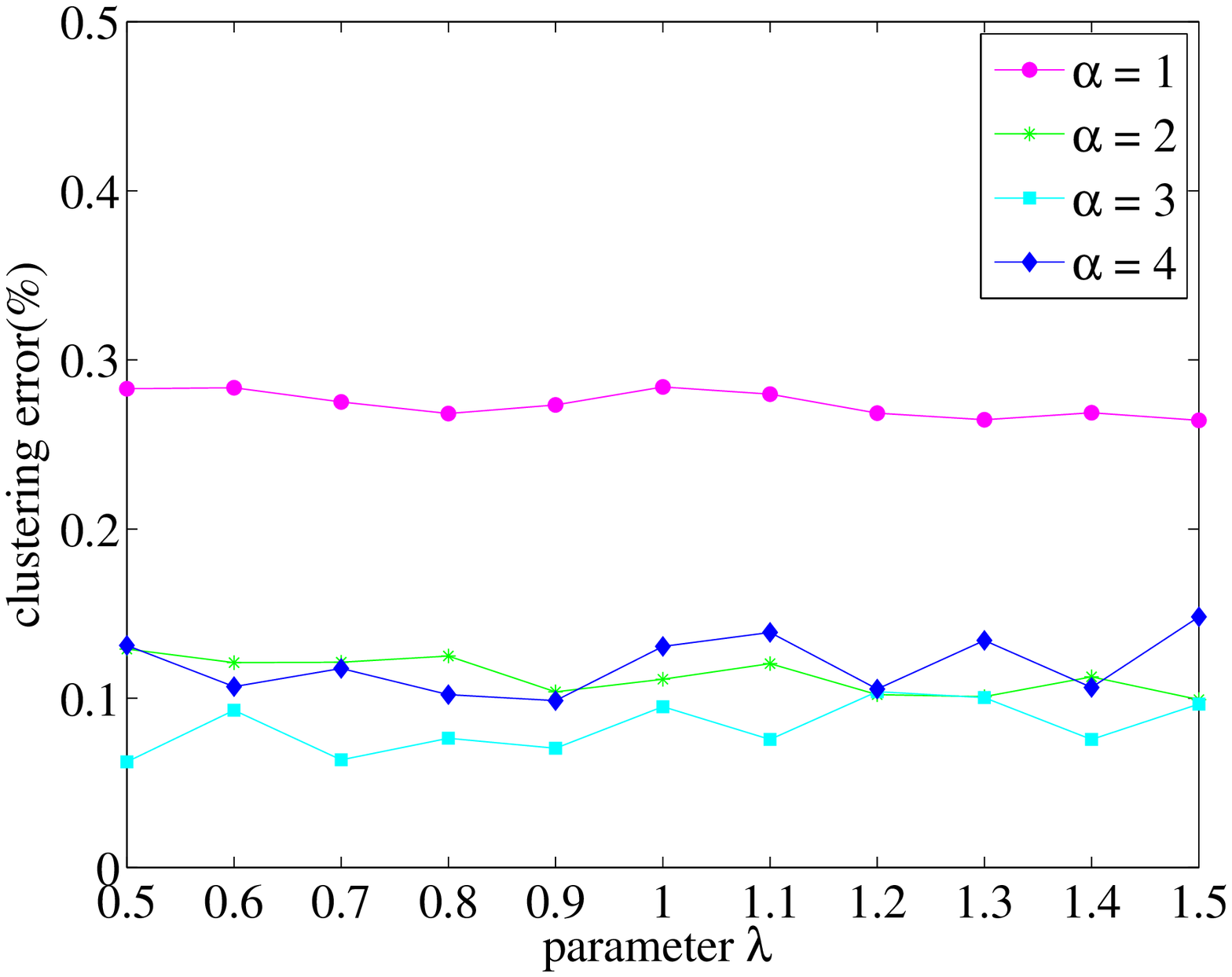}}
\end{minipage}%
\begin{minipage}[t]{0.24\linewidth}
\centering
\subfigure[$r=100$]{
\label{fig:facebyrp:c} 
\includegraphics[width=1\textwidth]{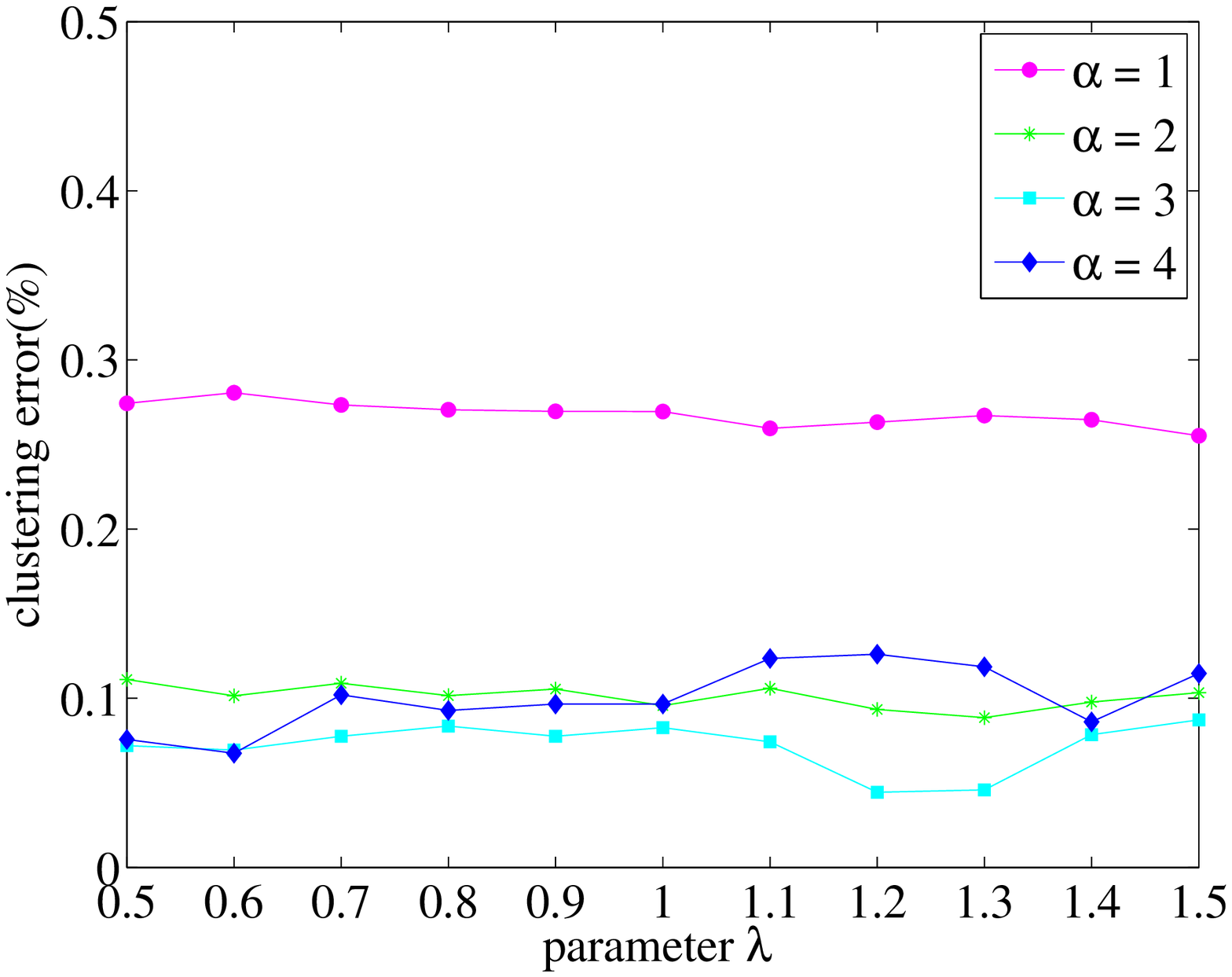}}
\end{minipage}%
\begin{minipage}[t]{0.24\linewidth}
\centering
\subfigure[$r=150$]{
\label{fig:facebyrp:d} 
\includegraphics[width=1\textwidth]{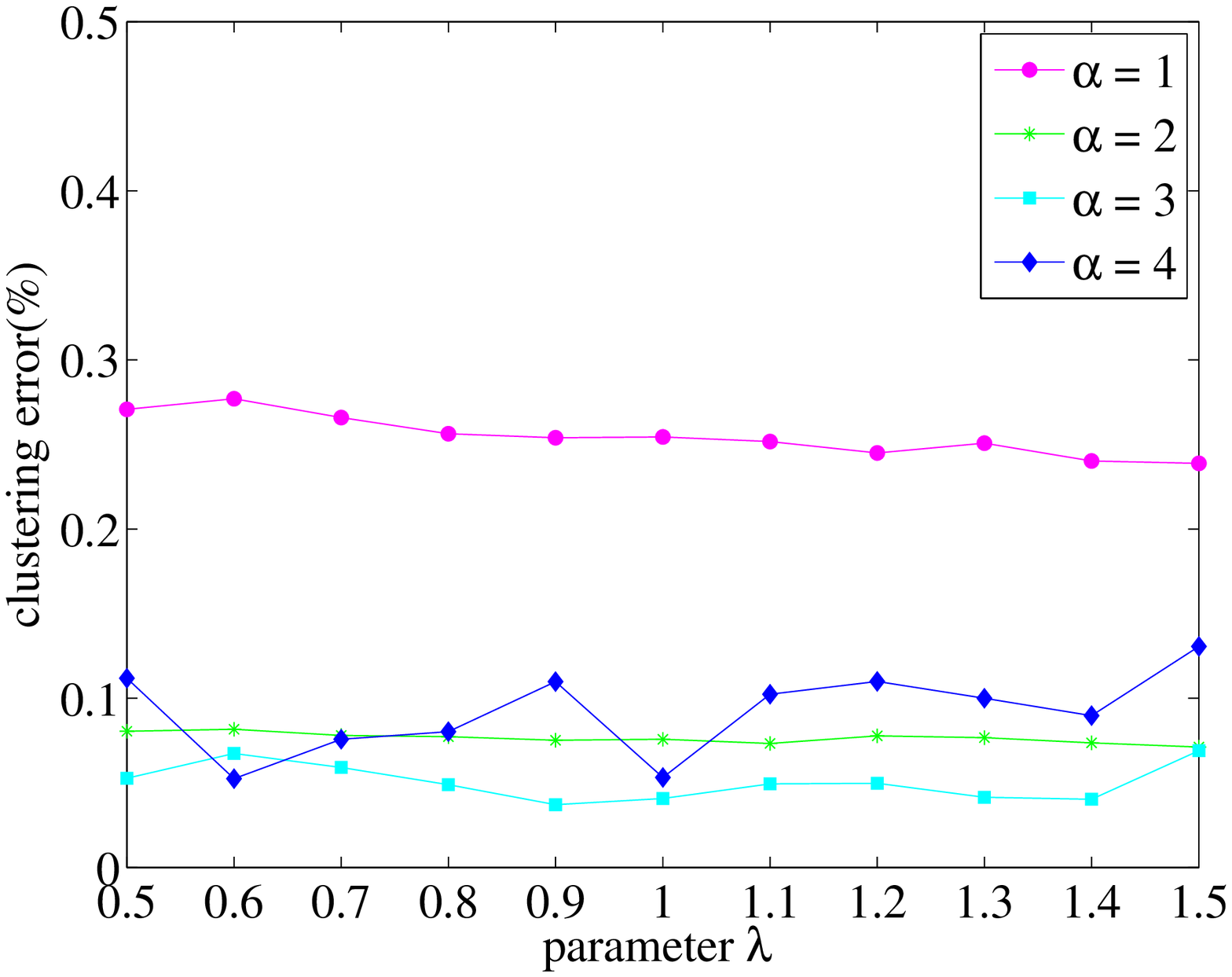}}
\end{minipage}%
\caption{Changes in clustering error when varying $\lambda$ and $\alpha$ under different $r$, by applying RP on face images of the first 10 classes in the Extended Yale Database B.}
\label{fig:facebyrp} 
\end{figure*}

We first examined the performance of these algorithms on the original data. Figures \ref{fig:facebypca} and \ref{fig:facebyrp} show the influence of parameters $\lambda$ and $\alpha$ for different $r$ values on the clustering errors of SLRR. Note that $r$ is the value of the reduced dimension after applying PCA or RP. Because the random projection matrix used in SLRR${_{RP}}$ is generated randomly, ten different random projection matrices are employed in the performance evaluation. The final clustering performance of SLRR${_{RP}}$ is computed by averaging the clustering error rates from these ten experiments. According to the Lambertian assumption mentioned above, the optimal value of $r$ is around 90 because of the 10 different subjects in the experiment. As shown in Figs. \ref{fig:facebypca:a} and \ref{fig:facebyrp:a}, a value of $r$ equal to $50$ results in inferior clustering performance. This implies that the reduced dimension information of face images, whose dimension is much less than the intrinsic dimension of face images, is not sufficient for low-rank representation to separate data from different subspaces. In contrast to the reduced dimension of the face images, a value of $r$ equal to or greater than $90$ leads to a significant performance improvement as shown in Figs.  \ref{fig:facebypca:b}$-$\ref{fig:facebypca:d} and \ref{fig:facebyrp:b}-\ref{fig:facebypca:d}. Therefore, parameter $r$ is closely related to the intrinsic dimension of high-dimensional data. In addition, SLRR${_{PCA}}$ and SLRR${_{RP}}$ seem to achieve better performance as $\alpha$ increases. For example, the clustering error of SLRR${_{PCA}}$ varies from $3.91\%$ to $4.38\%$ when $\lambda$ ranges from 20 to 100 with $\alpha=3$ in Fig. \ref{fig:facebypca:b}. On the contrary, the clustering error of SLRR${_{PCA}}$ varies from $20.47\%$ to $41.41\%$ when $\lambda$ ranges from 20 to 100 with $\alpha=1$ in Fig. \ref{fig:facebypca:b}. These comparisons can also be observed in Figs. \ref{fig:facebypca:c} $-$ \ref{fig:facebypca:d} and \ref{fig:facebyrp:b}$-$\ref{fig:facebypca:d}. However, SLRR${_{PCA}}$ and SLRR${_{RP}}$ cannot further improve the performance if $\alpha$ is too large (e.g., with $\alpha=4$ in Figs. \ref{fig:facebypca:b} $-$ \ref{fig:facebypca:d} and \ref{fig:facebyrp:b}$-$\ref{fig:facebypca:d}).

Table \ref{tb:face1} shows the face clustering results and computational cost of the different algorithms in the first experimental scenario. SLRR${_{PCA}}$ has better clustering performance and lower computational cost than the other algorithms. For example, the clustering errors of SLRR${_{PCA}}$ and SLRR${_{RP}}$ are 3.13\% and 4.44\%, respectively. SLRR${_{PCA}}$ improved the clustering accuracy by nearly 18\% compared with LRR. The improvement of SLRR${_{PCA}}$ and SLRR${_{RP}}$ indicates the importance of symmetric low-rank representation of high-dimensional data in the construction of the affinity graph matrix. From Table \ref{tb:face1}, it is clear that SLRR${_{PCA}}$, SLRR${_{RP}}$ and LRSC execute much faster than the other approaches. This is because they obtain a closed form solution of the low-rank representation on their corresponding optimization problems. SSC solves the ${l_1}$-norm minimization problem, while the optimization of LRR by inexact ALM requires hundreds of SVD computations before convergence. Hence, both of these incur a high computational cost. In SLRR${_{PCA}}$ and SLRR${_{RP}}$, collaborative representation with low-rank matrix recovery techniques into low-rank representation exhibits its efficiency by making use of the self-expressiveness property of the data.

\begin{table}[!htbp]
\small
\setlength{\abovecaptionskip}{0pt}
\setlength{\belowcaptionskip}{10pt}
\setlength{\tabcolsep}{1pt}
\centering
\caption{Clustering error (\%) and computation time (seconds) by applying different algorithms on the first 10 classes of the Extended Yale Database B.}
\label{tb:face1}
\begin{tabular}{|c|ccccccc|}
\hline
Algorithm & SLRR${_{PCA}}$ & SLRR${_{RP}}$ & LRRSC &  LRR  & SSC & LSA & LRSC \\
\hline
  error & \textbf{3.13} & 4.44 & 3.91 & 20.94 & 35 & 59.52 & 35.78 \\
\hline
  time & 35.26 & \textbf{34.81} & 115.63 & 103.66 & 54.06 & 91.51 & 35.29 \\
\hline
\end{tabular}
\end{table}

Finally, we explored the performance and robustness of these algorithms on a more challenging set of face images. Four artificial pixel corruption levels (10\%, 20\%, 30\%, and 40\%) were selected for the face images, and the locations of corrupted pixels were chosen randomly. To corrupt any chosen location, its observed value was replaced by a random number in the range [0, 1]. Some examples with 20\% pixel occlusions and their corrections are shown in Figures \ref{fig:yalebb} and \ref{fig:yalebc}, respectively. For a fair comparison, we applied RPCA to the corrupted face images for the other competing algorithms, where the RPCA parameter $\lambda$ ranged from 0.025 to 0.05. All experiments were repeated 10 times. Table \ref{tb:face11} shows the average clustering error. The results demonstrate that SLRR achieves a consistently high clustering accuracy when artificial pixel corruptions are relatively sparse, i.e., corruption percentages of 10\% and 20\%. As expected, the performance of SLRR deteriorates as the percentage of corruption increases. At corruption percentages of 30\% and 40\%, LRRSC obtains the highest clustering accuracy. The performance of SLRR degrades because the errors are no longer sparse during the low-rank matrix recovery algorithm, i.e., RPCA. LRR-based methods, such as SLRR, LRRSC, LRR, and LRSC perform better than the competing methods in all scenarios. This further highlights the benefit of estimating the underlying subspaces using the low-rank criterion. Compared with the other competing methods, SLRR and LRRSC are slightly more stable when the given data is corrupted by gross errors.

\begin{table}[!htbp]
\small
\setlength{\abovecaptionskip}{0pt}
\setlength{\belowcaptionskip}{10pt}
\setlength{\tabcolsep}{1pt}
\centering
\caption{Clustering error (\%) by applying different algorithms on the first 10 classes of the Extended Yale Database B with four artificial pixel corruption levels.}
\label{tb:face11}
\begin{tabular}{|c|cccccc|}
\hline
Corruption ratio (\%) & SLRR${_{RPCA}}$ & LRRSC &  LRR  & SSC & LSA & LRSC \\
\hline
  10 & \textbf{9.23} & 12.16 & 21.38 & 32.84 & 60.86 & 16.22 \\
\hline
  20 & \textbf{10.34} & 12.25 & 24.77 & 39.44 & 62.89 & 17.47 \\
\hline
  30 & 13.69 & \textbf{12.79} & 30.44 & 43.84 & 61.58 & 17.2 \\
\hline
  40 & 14.59 & \textbf{12.23} & 31.72 & 48.95 & 64.98 & 20.72 \\
\hline
\end{tabular}
\end{table}

\subsubsection{Second scenario for face clustering}

We used the experimental settings from \cite{Elhamifar2013SSC}. The 38 subjects were divided into four groups as follows: subjects 1 to 10, 11 to 20, 21 to 30, and 31 to 38 corresponding to the four different groups. All choices of $n \in \{ 2,3,5,8,10\}$ were considered for each of the first three groups, and all choices of $n \in \{ 2,3,5,8\} $ were considered for the last group. Finally, we applied each algorithm to each choice (i.e., each set of $n$ subjects) in the experiments, and the mean and median subspace clustering errors for different numbers of subjects were computed.

\begin{table}[!htbp]
\small
\setlength{\abovecaptionskip}{0pt}
\setlength{\belowcaptionskip}{10pt}
\setlength{\tabcolsep}{2pt}
\centering
\caption{Average clustering error (\%) for different numbers of subjects on the Extended Yale B database.}
\label{tb:face2}
\begin{tabular}{|c|ccccccc|}
\hline
Algorithm & SLRR${_{PCA}}$ & SLRR${_{RP}}$ & LRRSC & LRR & SSC & LSA & LRSC \\
\hline
2 Subjects & & & & & & &\\
Mean & \textbf{1.29} & 4.81 & 1.78 & 2.54 & 1.86 & 41.97 & 4.25 \\
Median & 0.78 & 2.34 & 0.78 & 0.78 & \textbf{0} & 47.66 & 3.13 \\
\hline
3 Subjects & & & & & & &\\
Mean & \textbf{1.94} & 6.18 & 2.61 & 4.23 & 3.24 & 56.62 & 6.07 \\
Median & 1.56 & 4.17 & 1.56 & 2.6 & \textbf{1.04} & 61.98 & 5.73 \\
\hline
5 Subjects & & & & & & &\\
Mean & \textbf{2.72} & 6.03 & 3.19 & 6.92 & 4.33 & 59.29 & 10.19 \\
Median & \textbf{2.5} & 4.98 & 2.81 & 5.63 & 2.82 & 56.25 & 7.5 \\
\hline
8 Subjects & & & & & & &\\
Mean & \textbf{3.21} & 7.42 & 4.01 & 13.62 & 5.87 & 57.17 & 23.65 \\
Median & \textbf{2.93} & 4.98 & 3.13 & 9.67 & 4.49 & 59.38 & 27.83 \\
\hline
10 Subjects & & & & & & &\\
Mean & 3.49 & \textbf{3.44} & 3.7 & 14.58 & 7.29 & 59.38 & 31.46 \\
Median & \textbf{2.81} & 3.28 & 3.28 & 16.56 & 5.47 & 60.94 & 28.13 \\
\hline
\end{tabular}
\end{table}

\begin{figure}[!htb]
\centering
\includegraphics[width=7cm]{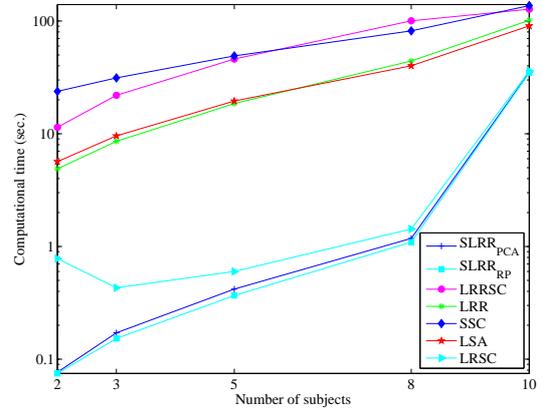}
\caption{Average computation time (seconds) for different numbers of subjects on the Extended Yale B database.}
\label{fig:facetime2} 
\end{figure}

Table \ref{tb:face2} shows the clustering results for the various algorithms using different numbers of subjects. The SLRR${_{PCA}}$ algorithm almost consistently obtained lower mean clustering errors than the other algorithms for a varying number of subjects. This confirms that our proposed method is very effective and robust against a varying number of subjects with respect to face clustering. We also observed that the clustering performance by SLRR${_{RP}}$ outperforms that of SLRR${_{PCA}}$ by a very small margin with 10 subjects. However, SLRR${_{RP}}$ performs worse than SLRR${_{PCA}}$ as well as LRRSC and SSC when the number of subjects is less than 10. However, we also see that increasing the number of clusters of SLRR${_{RP}}$ achieved a greater improvement compared with LRR. This phenomenon can be explained as follows. On the one hand, the clustering results of SLRR${_{RP}}$ are effected largely by the randomly generated project matrix. On the other hand, what we emphasize is the importance of the determination of low-rank matrix recovery techniques for an alternative low-rank matrix. Moreover, we also compared the computational costs shown in Fig. \ref{fig:facetime2}. The computational costs of  SLRR${_{PCA}}$ and SLRR${_{RP}}$ are very similar, and only slightly better than LRSC. LRSC also had relatively low computational time at the expense of degraded performance in the experiments. In fact,   SLRR${_{PCA}}$ and SLRR${_{RP}}$ achieved high efficiency owing to completely avoiding the iterative SVD computation. Both of these run much faster than the other algorithms, e.g., LRR, LRRSC, SSC, and LSA.

Figure \ref{fig:affinity} depicts seven representative examples of the affinity graph matrix produced by the different algorithms for the Extended Yale Database B with five subjects. Clearly there are five diagonal blocks in each affinity graph. The smaller the number of non-zero elements lying outside the diagonal blocks is, the more accurate are the clustering results in spectral clustering. It is clear from Fig. \ref{fig:affinity} that the affinity graph matrix produced by SLRR${_{PCA}}$ has a distinct block-diagonal structure as is the case for LRRSC. This shows why SLRR${_{PCA}}$ outperforms the other algorithms.

\begin{figure*}[!htbp]
\begin{minipage}[t]{0.14\linewidth}
\centering
\subfigure[SLRR${_{PCA}}$]{
\label{fig:affinity:c}
\includegraphics[width=1\textwidth]{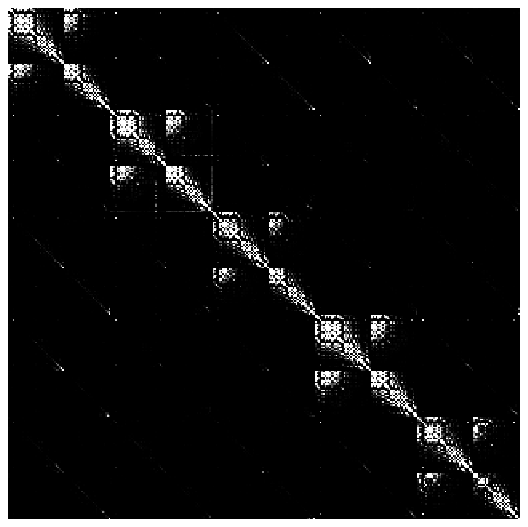}}
\end{minipage}%
\begin{minipage}[t]{0.14\linewidth}
\centering
\subfigure[SLRR${_{RP}}$]{
\label{fig:affinity:c}
\includegraphics[width=1\textwidth]{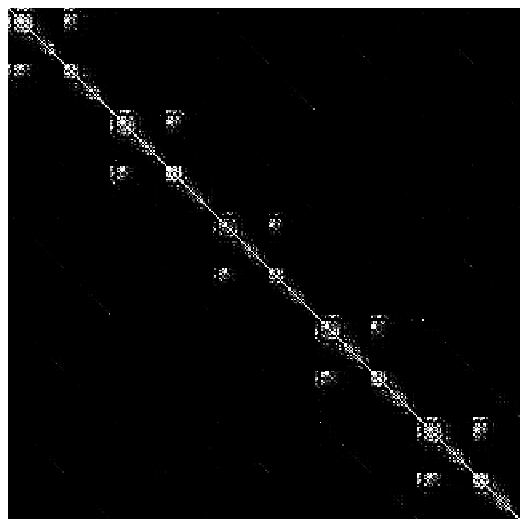}}
\end{minipage}%
\begin{minipage}[t]{0.14\linewidth}
\centering
\subfigure[LRRSC]{
\label{fig:affinity:a}
\includegraphics[width=1\textwidth]{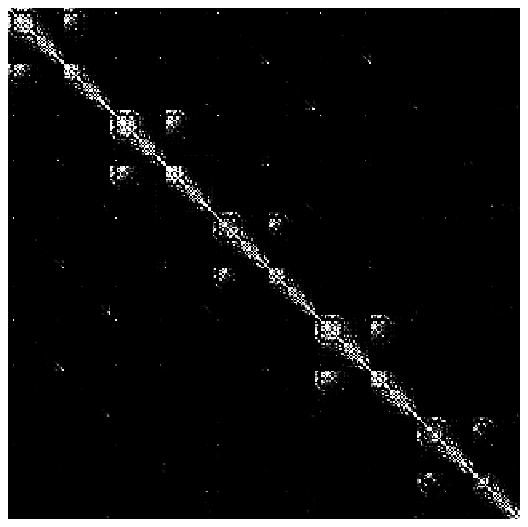}}
\end{minipage}%
\begin{minipage}[t]{0.14\linewidth}
\centering
\subfigure[LRR]{
\label{fig:affinity:b}
\includegraphics[width=1\textwidth]{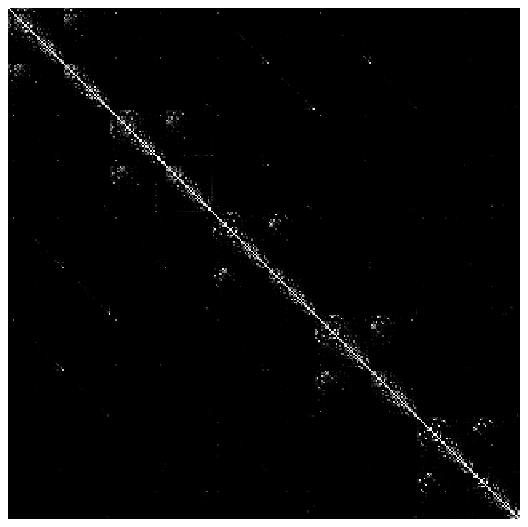}}
\end{minipage}%
\begin{minipage}[t]{0.14\linewidth}
\centering
\subfigure[SSC]{
\label{fig:affinity:d}
\includegraphics[width=1\textwidth]{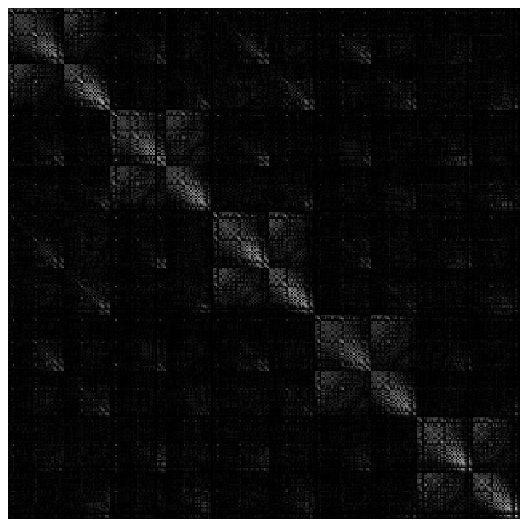}}
\end{minipage}%
\begin{minipage}[t]{0.14\linewidth}
\centering
\subfigure[LSA]{
\label{fig:affinity:e}
\includegraphics[width=1\textwidth]{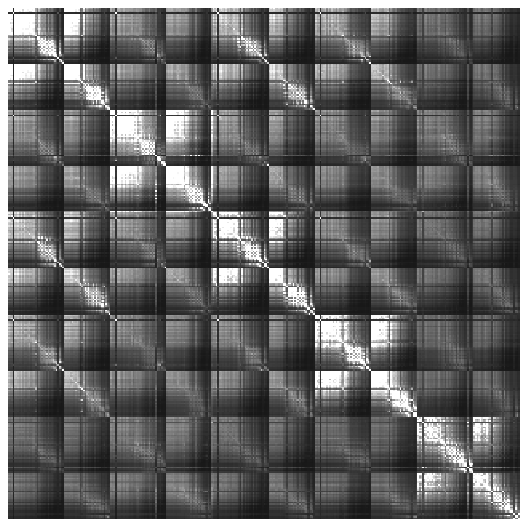}}
\end{minipage}%
\begin{minipage}[t]{0.14\linewidth}
\centering
\subfigure[LRSC]{
\label{fig:affinity:f}
\includegraphics[width=1\textwidth]{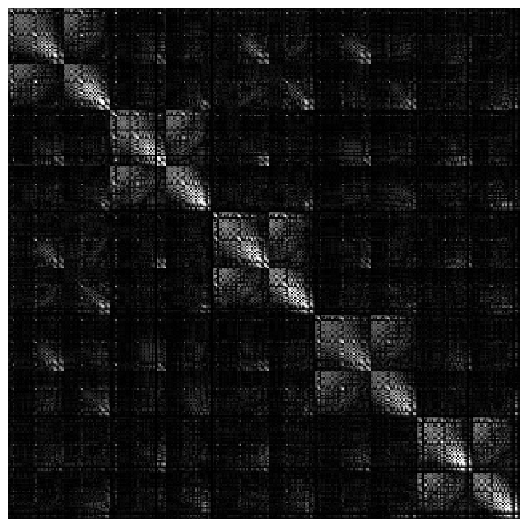}}
\end{minipage}%
\caption{Representative examples of the affinity graph matrix produced when using different algorithms for the Extended Yale Database B with five subjects.}
\label{fig:affinity}
\end{figure*}

\begin{figure*}[!htbp]
\centering
\includegraphics[width=0.24\textwidth]{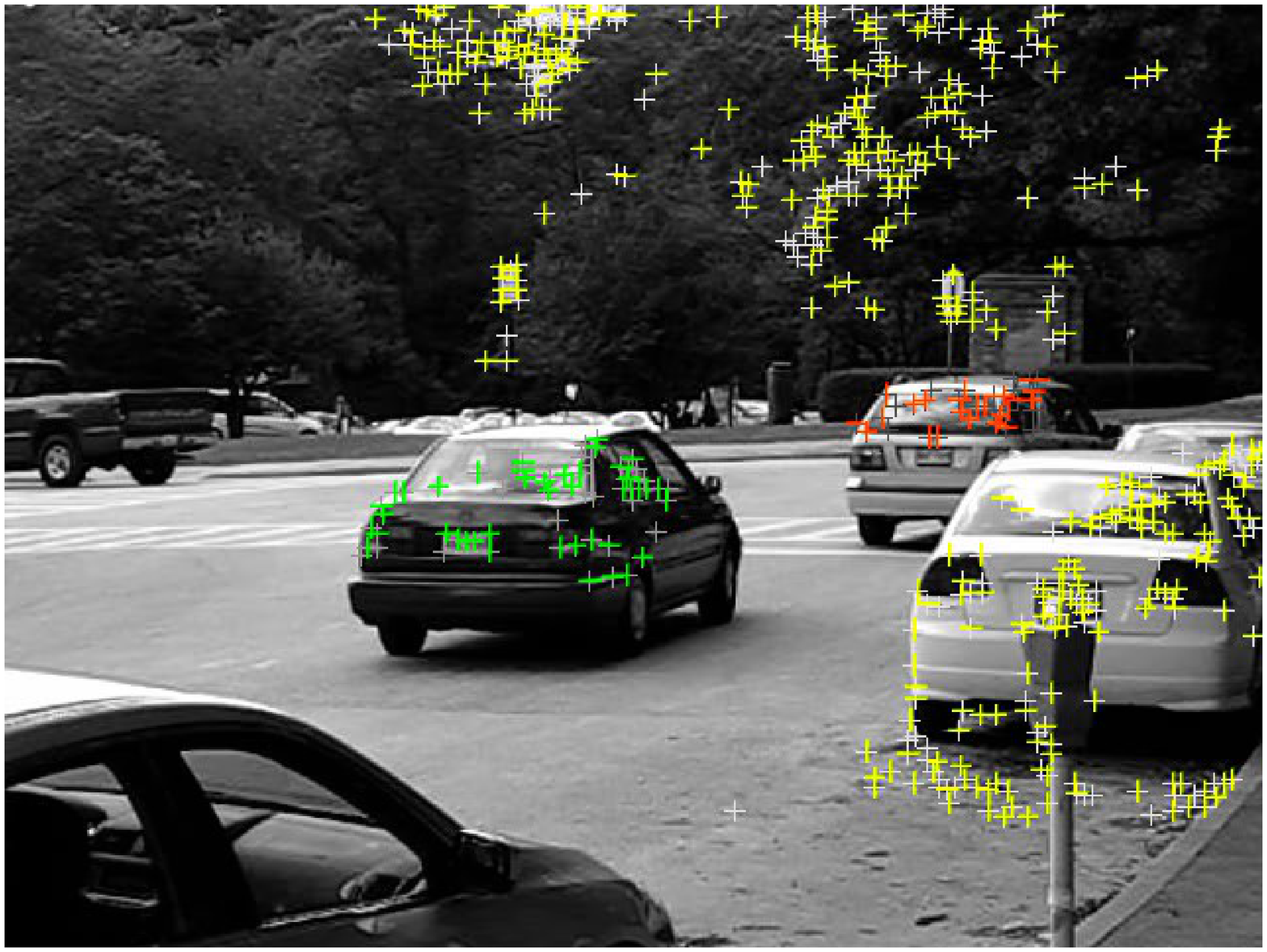}
\includegraphics[width=0.24\textwidth]{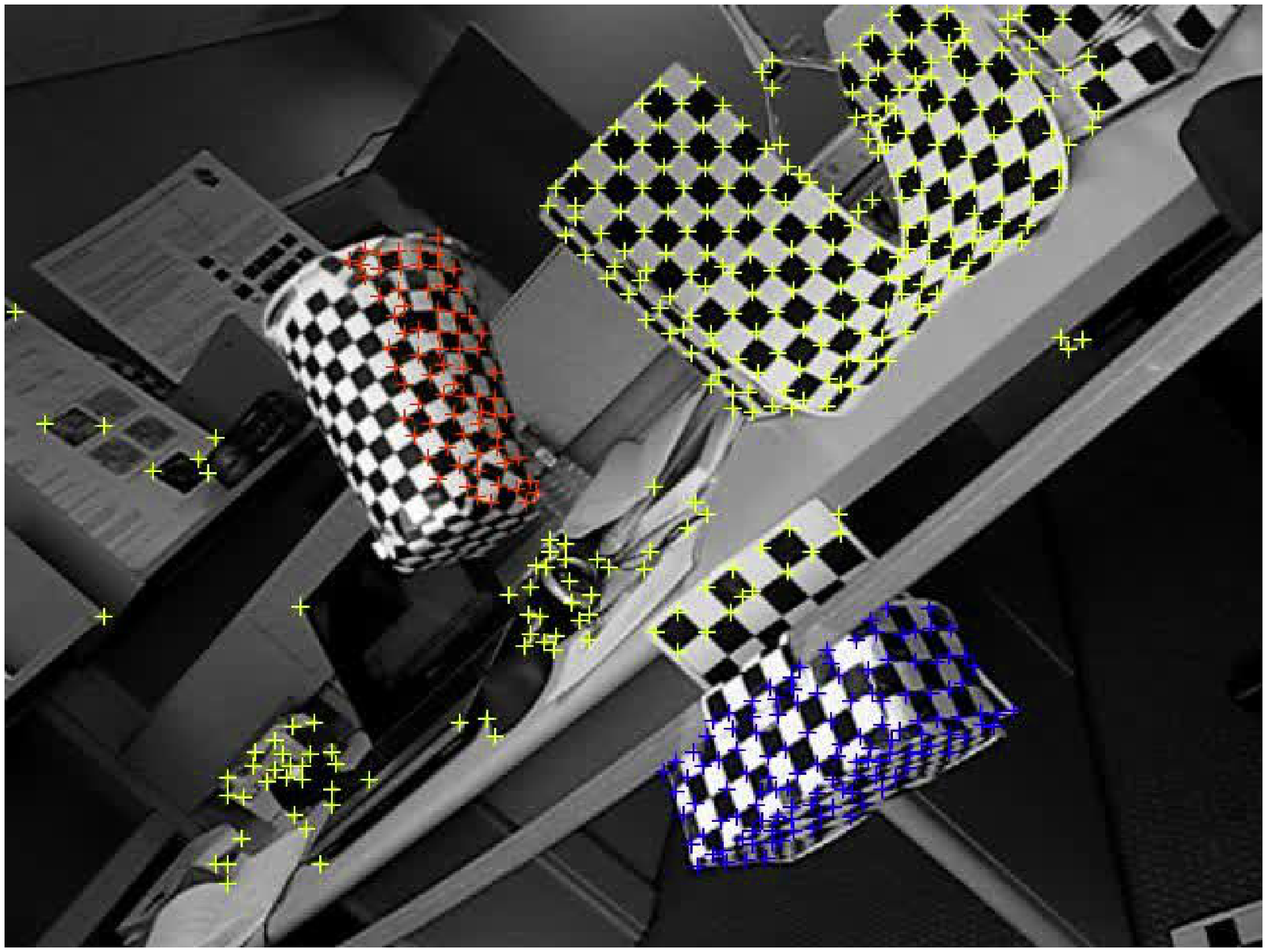}
\includegraphics[width=0.24\textwidth]{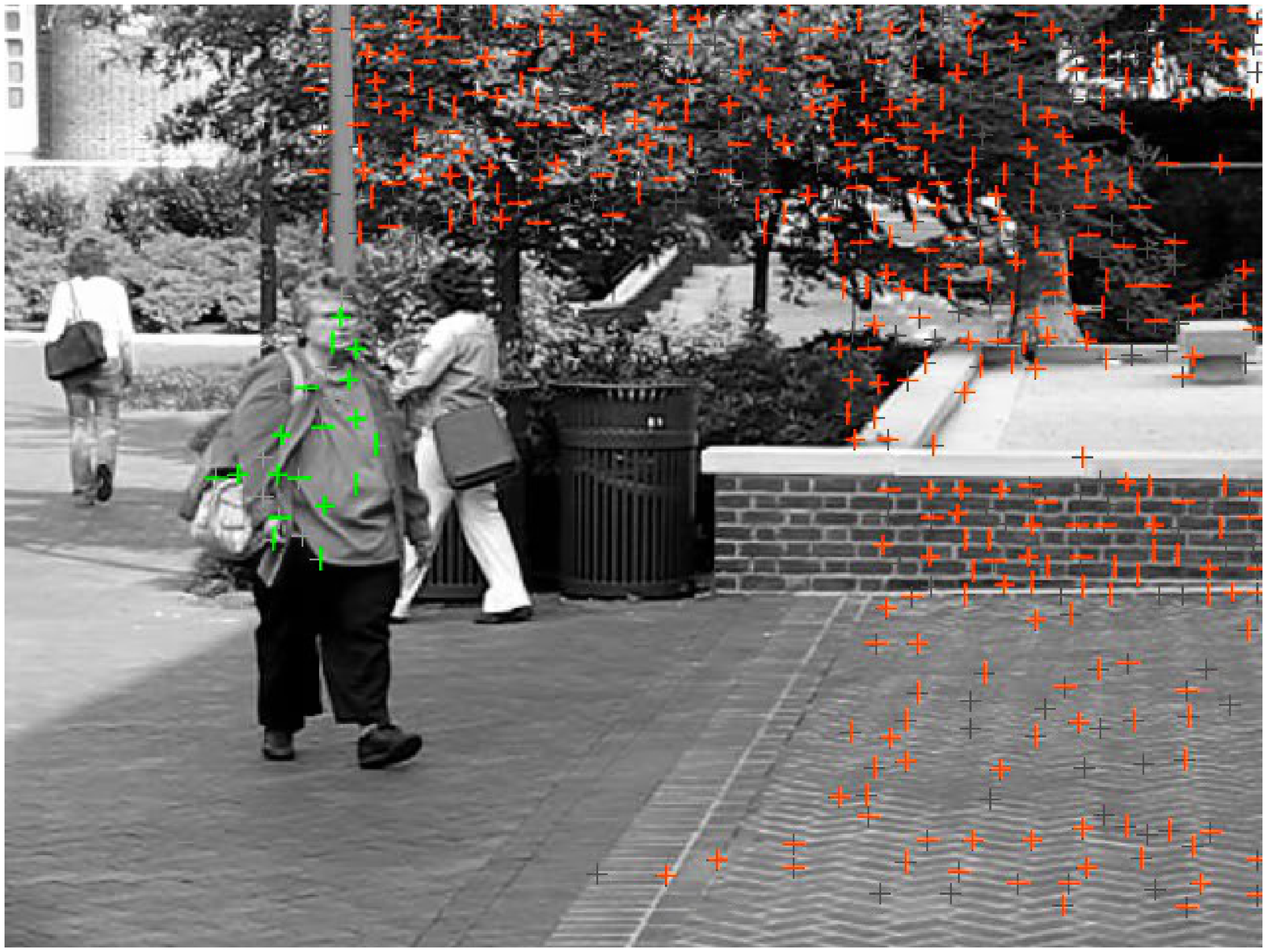}
\includegraphics[width=0.24\textwidth]{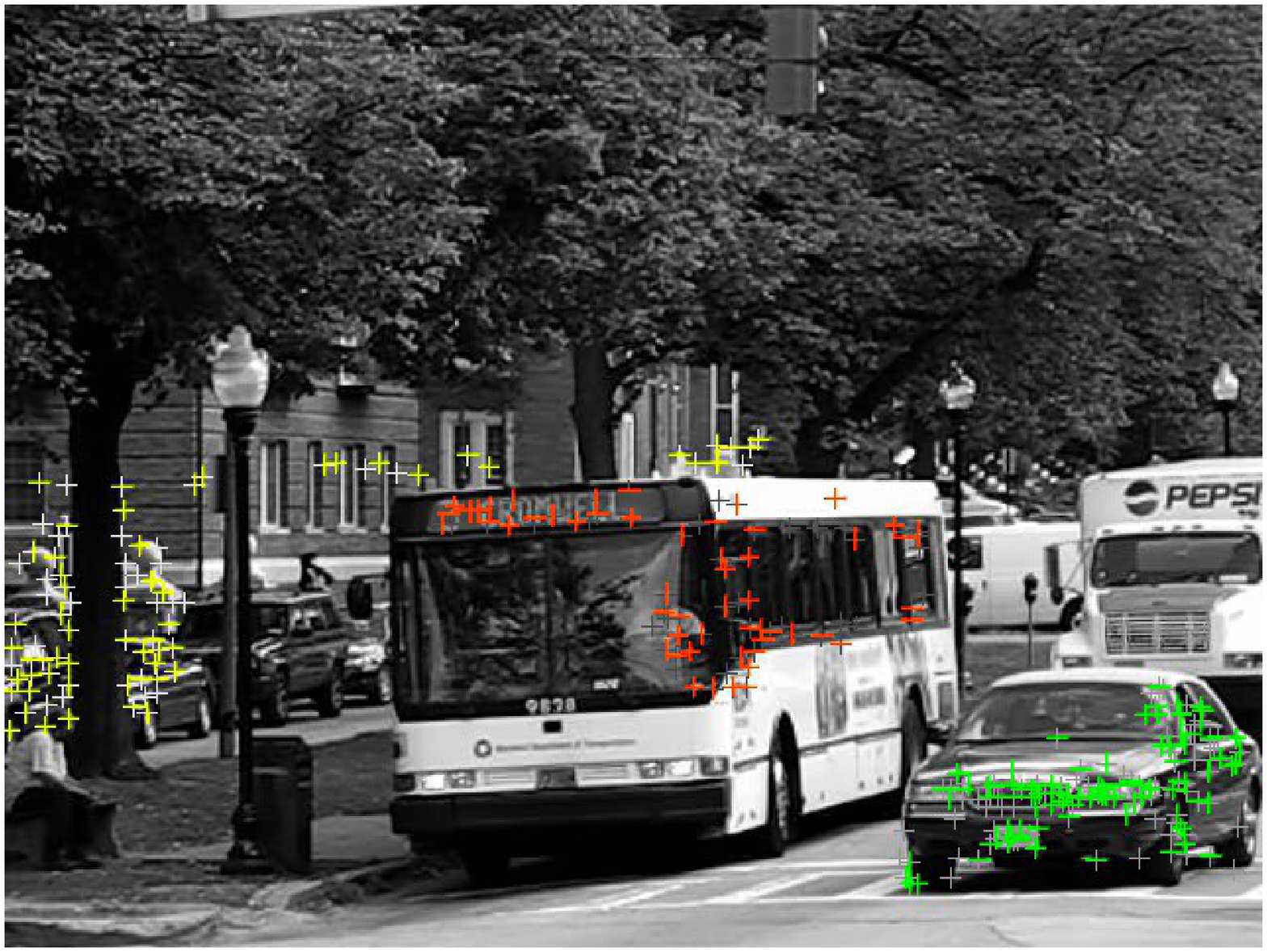}
\caption{Example frames from four video sequences of the Hopkins 155 database with traced feature points.}
\label{fig:h155example} 
\end{figure*}

\subsection{Experiments on motion segmentation}

In this subsection we discuss applying SLRR to the Hopkins 155 database. The task of motion segmentation involves segmenting tracked feature point trajectories of multiple rigidly moving objects into their corresponding motions in a video sequence. Each video sequence is a sole subspace segmentation task. There are 156 video sequences of two or three motions in the Hopkins 155 database. As pointed out in \cite{Boult1991Factor}, the tracked feature point trajectories for a single motion lie in a low-dimensional subspace. Therefore, the motion segmentation problem is equivalent to the problem of subspace clustering.

For each video sequence, tracked feature point trajectories were extracted automatically and the original data were almost noise-free, i.e., low-rank. Hence, we designed two experiments to evaluate the performance of the proposed SLRR in motion segmentation. First, we used the original tracked feature point trajectories associated with each motion to validate SLRR. Next, we used PCA to project the original data onto a  -dimensional subspace, where   is the number of motions in each video sequence. Note that both scenarios were implemented in an affine subspace, thereby ensuring that the sum of the feature point trajectory coefficients was $1$.

Figures \ref{fig:h155:a} and \ref{fig:h155:b} show the influence of parameters $\lambda$ and $\alpha$ under two experimental settings of the Hopkins 155 database on the average clustering error of SLRR. It is clear that increasing $\alpha$ from 1 to 2 produced higher clustering performance. For example, the clustering error varies from 0.88\% to 4.22\% while $\lambda$ ranges from ${\rm{1}}{{\rm{e}}^{ - 3}}$ to ${\rm{1}}{{\rm{e}}^{ - 2}}$ with $\alpha = 2$ in Fig. \ref{fig:h155:a}. If $\lambda$ ranges from ${\rm{4}}{{\rm{e}}^{ -3}}$ to ${\rm{7}}{{\rm{e}}^{ - 3}}$, the clustering error appears to change only slightly, varying from 0.88\% to 1.04\% in Fig. \ref{fig:h155:a}. However, SLRR suffers from a decline in clustering performance when $\alpha$ continues to increase from 2 to 4 in Fig. \ref{fig:h155:a}. We observed a similar influence of parameters $\lambda$ and $\alpha$ in Fig. \ref{fig:h155:b}. This implies that the clustering performance of SLRR on the Hopkins 155 database remains relatively stable for a large range of $\lambda$ with $\alpha=2$.

\begin{figure}[!htbp]
\begin{minipage}[t]{0.5\linewidth}
\centering
\subfigure[]{
\label{fig:h155:a} 
\includegraphics[width=1\textwidth]{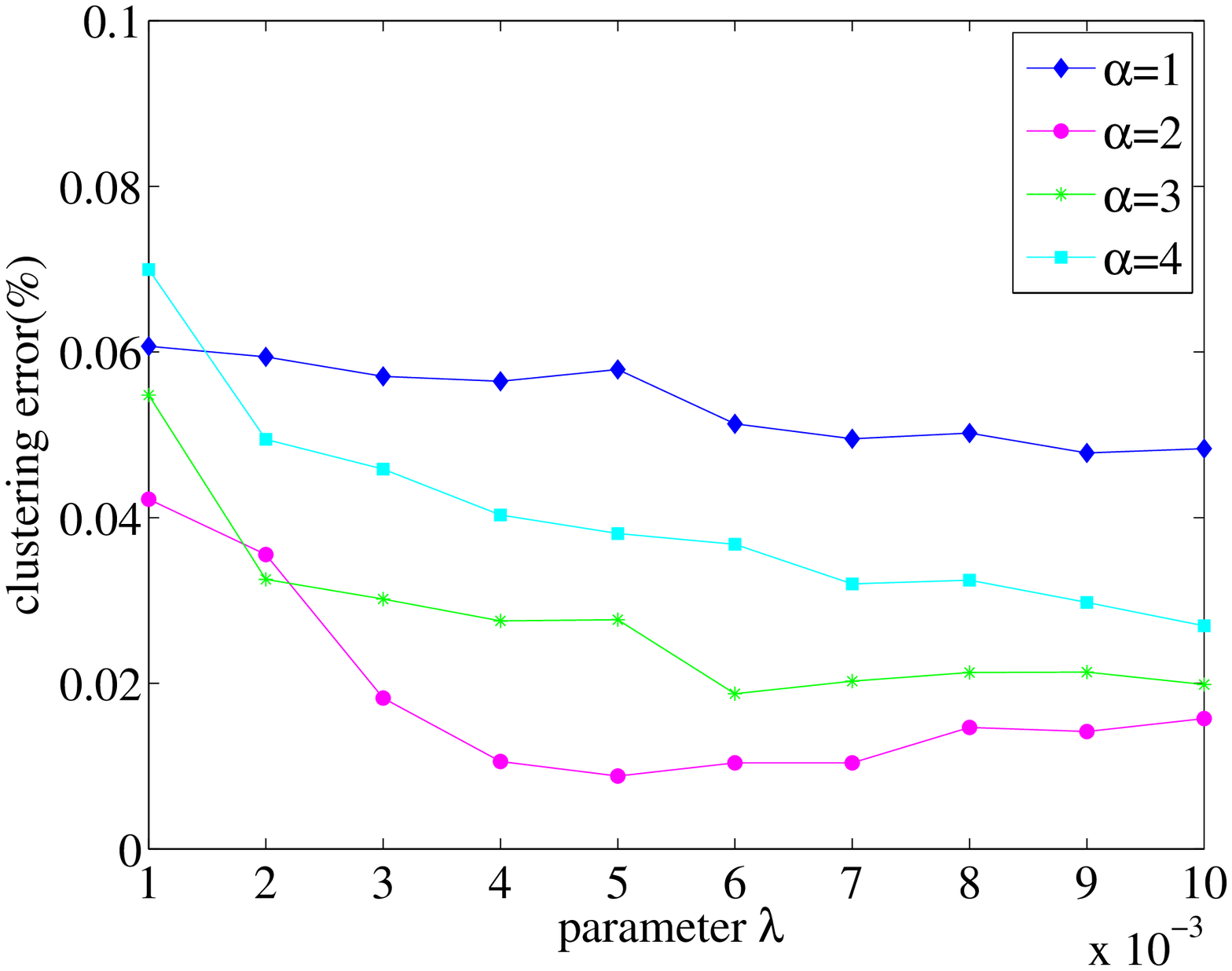}}
\end{minipage}%
\begin{minipage}[t]{0.5\linewidth}
\centering
\subfigure[]{
\label{fig:h155:b} 
\includegraphics[width=1\textwidth]{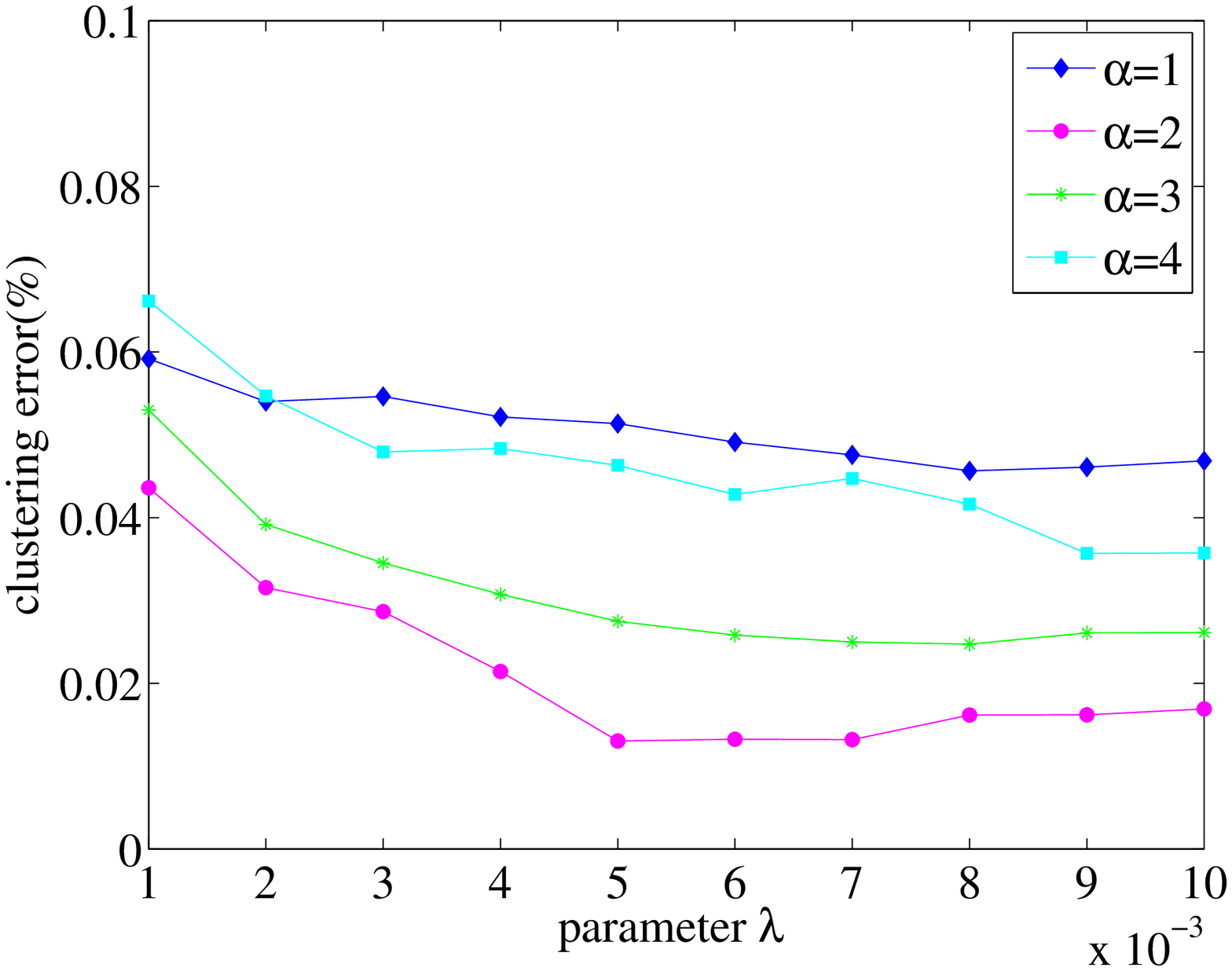}}
\end{minipage}
\caption {Influences of the parameter $\lambda$ of SLRR. (a) The average clustering error of SLRR on the Hopkins 155 database with the $2F$-dimensional data points. (b) The average clustering error of SLRR on the Hopkins 155 database with the 4n-dimensional data by applying PCA.}
\label{fig:h155} 
\end{figure}

Tables \ref{h155tb1} and \ref{h155tb2} show the average clustering errors of the different algorithms on two experimental settings of the Hopkins 155 database. SLRR obtained 0.88\% and 1.3\% clustering errors for the two experimental settings. In both experimental settings, SLLR significantly outperformed the other algorithms. We used the normalization step of symmetric low-rank representation to improve the clustering performance to further seek an affinity matrix. Under the same parameter settings, we also report the clustering error of SLRR within parentheses without the normalization step in Tables \ref{h155tb1} and \ref{h155tb2}. The normalization step helps improve the clustering results. Compared with LRR, SLRR achieved 0.83\% and 0.87\% improvement on clustering errors for the two settings, respectively. The improvement comes from the advantages of the compactness of the symmetric low-rank representation. LRR has lower errors than SSC owing to the post-processing of its coefficient matrix. This also confirms the necessity of exploiting the structure of the low-rank representation for an affinity graph matrix. Besides, LRSC still has higher errors than the other algorithms in both experiments.

\begin{table}[!htbp]
\small
\setlength{\abovecaptionskip}{0pt}
\setlength{\belowcaptionskip}{10pt}
\setlength{\tabcolsep}{1pt}
\centering
\caption{Average clustering error (\%) and mean computation time (seconds) when applying the different algorithms to the Honkins 155 database, with the $2F$-dimensional data points.}
\label{h155tb1}
\begin{tabular}{|c|c|c|c|c|c|}
\hline
\multirow{2}{*}{Algorithm} & \multicolumn{4}{|c|}{Error} & \multirow{2}{*}{Time} \\
\cline{2-5}
 & mean & median & std. & max. &\\
\hline
  SLRR & \textbf{0.88} (3)  & \textbf{0 (0)}  & \textbf{3.63} (9.33) & 38.06 (49.25) & \textbf{0.09 (0.09)} \\
  LRRSC & 1.5  & \textbf{0}  & 4.36  & \textbf{33.33}  & 4.71 \\
  LRR & 1.71  & \textbf{0} & 4.86  & \textbf{33.33} & 1.29 \\
  SSC & 2.23  & \textbf{0}  & 7.26  & 47.19 & 1.02 \\
  LSA & 11.11  & 6.29 & 13.04  & 51.92 & 3.44 \\
  LRSC & 4.73  & 0.59  & 8.8  & 40.55 & 0.14 \\
\hline
\end{tabular}
\end{table}

\begin{table}[!htbp]
\small
\setlength{\abovecaptionskip}{0pt}
\setlength{\belowcaptionskip}{10pt}
\setlength{\tabcolsep}{1pt}
\centering
\caption{Average clustering error (\%) and mean computation time (seconds) when applying the different algorithms to the Honkins 155 database, with the $4n$-dimensional data points obtained using PCA.}
\label{h155tb2}
\begin{tabular}{|c|c|c|c|c|c|}
\hline
\multirow{2}{*}{Algorithm} & \multicolumn{4}{|c|}{Error} & \multirow{2}{*}{Time} \\
\cline{2-5}
 & mean & median & std. & max. &\\
\hline
  SLRR & \textbf{1.3} (2.42)  & \textbf{0 (0)}  & \textbf{5.1} (8.14)  & \textbf{42.16} (49.25) & \textbf{0.07} (0.08) \\
  LRRSC & 1.56  & \textbf{0}  & 5.48  & 43.38  & 4.62\\
  LRR & 2.17  & \textbf{0} & 6.58 & 43.38 & 0.69  \\
  SSC & 2.47  & \textbf{0}  & 7.5  & 47.19 & 0.93 \\
  LSA  & 4.7  & 0.6  & 10.2  & 54.51 & 3.35 \\
  LRSC & 4.89  & 0.63  & 8.91  & 40.55 & 0.13 \\
\hline
\end{tabular}
\end{table}

The computational cost of SLRR is much lower than that of the other algorithms owing to its closed form solution. High clustering performance can also be obtained when the original data are used directly in SLRR. This phenomenon occurs because most sequences are clean, i.e., low-rankness property. However, this does not deny the importance of pursing an alternative low-rank matrix by low-rank matrix recovery techniques. Clean data are not easily obtained because of noise or corruption in real observations.

\subsection{Discussion}
\label{sec:Discussion}
Our experiments show that the performance of SLRR and LRR  differs, with a relative clustering error reduction of more than 10\% in some cases. In what follows, we discuss the connection between SLRR and LRR.

First, LRR not only seeks the best low-rank representation of high-dimensional data for matrix recovery, but also recovers the true subspace structures. Contrarily, SLRR focuses only on how to recover the true subspace structures. Generally, ${z_{ij}}$ differs from ${z_{ji}}$ in the low-rank representation $Z$ obtained by LRR, where ${z_{ij}}$ or ${z_{ji}}$ depicts the membership between data points $i$ and $j$. LRR constructs the affinity for the spectral clustering input using a symmetrization step of the low-rank representation results, i.e., ${Z^{\ast}} = \left| Z \right| + \left| {{Z^T}} \right|$. Evaluating the membership between data points, however, is not good, because LRR attempts to enforce symmetry of the affinity using this trick, whereas SLRR directly models the symmetric low-rank representation, thereby ensuring weight consistency for each pair of data points. The symmetric low-rank representation given by SLRR effectively preserves the subspace structures of high-dimensional data.

Second, SLRR further exploits the intrinsically geometrical structure of the membership of data points preserved in the symmetric low-rank representation. Note that the mechanism for exploiting this has been elaborated in our previous work, i.e., LRRSC \cite{Chen2014SC}. In other words, SLRR makes full use of the angular information of the principal directions of the symmetric low-rank representation so that highly correlated data points of subspaces are clustered together. This is a critical step in calculating the membership between data points. Fig. \ref{fig:affinity} shows that the block-diagonal structure of the affinity produced by SLRR is more distinct and compact than that obtained by LRR. The experimental results demonstrate that this significantly improves subspace clustering performance.

Finally, SLRR provides a more flexible model of the low-rank representation. SLRR integrates the collaborative representation combined with low-rank matrix recovery techniques into a low-rank representation with respect to various types of noise, e.g., Gaussian noise and arbitrary sparse noise. Additionally, it avoids iterative SVD operations while learning a symmetric low-rank representation. However, we need to emphasize that SLRR does not pursue the lowest-rank representation of data for evaluating the membership between data points. Strictly speaking, it does not make sense to pursue only the lowest-rank representation of data. Let us consider the LRR model again. The corresponding optimal solution can be obtained for an arbitrary value of parameter $\lambda$ in problem \eqref{eq:lrr2}. Obviously, we cannot determine which low-rank matrix of the optimal solution is desirable without prior knowledge of the data set. Hence, it is reasonable that SLRR tries to obtain a symmetric low-rank representation of the data. This also explains why we use the constraint, $rank(A) \le r$, to guarantee the low-rank property of the symmetric representation of the data in problem \eqref{eq:SLRR} or \eqref{eq:TheFinalSLRR}. In fact, in our experiments, we have also discussed some of the detail involved in estimating the rank of a data matrix of images of various examples, for example, images of an individual¡¯s face and handwritten images of a digit.

\section{Conclusions}
\label{sec:Conclusions}

In this paper we presented a method called SLRR, which considers collaborative representation combined with low-rank matrix recovery techniques to create a low-rank representation for robust subspace clustering. Unlike time-consuming SVD operations in many existing low-rank representation based algorithms, SLRR involves learning a symmetric low-rank representation in a closed form solution by solving the symmetric low-rank optimization problem, which greatly reduces computational cost in practical applications. Experimental results on benchmark databases demonstrated that SLRR is efficient and effective for subspace clustering compared with several state-of-the-art subspace clustering algorithms.

SLRR is a simple and effective method, which is considered an improvement over our previously proposed LRRSC \cite{Chen2014SC}. However, several problems remain to be solved. In the implementation of SLRR, it is important how to introduce low-rank matrix recovery algorithms, because a proper alternative low-rank matrix may significantly improve the subspace clustering performance. In addition, the determination of the parameter $r$ for pursing an alternative low-rank matrix by low-rank matrix recovery or feature extraction is also an intractable problem. Moreover, it is difficult to estimate a suitable value of ${\lambda}$ without prior knowledge. In future work, we will investigate these problems for practical applications.

\section*{Acknowledgements}

The authors would like to thank the anonymous reviewers for their valuable comments and suggestions. This research was supported by the National Basic Research Program of China (973 Program) under Grant 2011CB302201 and the National Science Foundation of China under Grant 61303015.

\bibliographystyle{model1-num-names}
\bibliography{slrr}

\begin{thebibliography}{49}
\expandafter\ifx\csname natexlab\endcsname\relax\def\natexlab#1{#1}\fi
\providecommand{\url}[1]{\texttt{#1}}
\providecommand{\href}[2]{#2}
\providecommand{\path}[1]{#1}
\providecommand{\DOIprefix}{doi:}
\providecommand{\ArXivprefix}{arXiv:}
\providecommand{\URLprefix}{URL: }
\providecommand{\Pubmedprefix}{pmid:}
\providecommand{\doi}[1]{\href{http://dx.doi.org/#1}{\path{#1}}}
\providecommand{\Pubmed}[1]{\href{pmid:#1}{\path{#1}}}
\providecommand{\bibinfo}[2]{#2}
\ifx\xfnm\relax \def\xfnm[#1]{\unskip,\space#1}\fi
\bibitem[{Eldar and Mishali(2009)}]{Eldar2009RRS}
\bibinfo{author}{Y.~C. Eldar}, \bibinfo{author}{M.~Mishali},
\newblock \bibinfo{title}{Robust recovery of signals from a structured union of
  subspaces},
\newblock \bibinfo{journal}{IEEE Trans. Inf. Theor.} \bibinfo{volume}{55}
  (\bibinfo{year}{2009}) \bibinfo{pages}{5302--5316}.
\bibitem[{Liu et~al.(2013)Liu, Lin, Yan, Sun, Yu, and Ma}]{Liu2010LRR}
\bibinfo{author}{G.~Liu}, \bibinfo{author}{Z.~Lin}, \bibinfo{author}{S.~Yan},
  \bibinfo{author}{J.~Sun}, \bibinfo{author}{Y.~Yu}, \bibinfo{author}{Y.~Ma},
\newblock \bibinfo{title}{Robust recovery of subspace structures by low-rank
  representation},
\newblock \bibinfo{journal}{IEEE Trans. Pattern Anal. and Mach. Intell.}
  \bibinfo{volume}{35} (\bibinfo{year}{2013}) \bibinfo{pages}{171--184}.
\bibitem[{Liu et~al.(2010)Liu, Lin, and Yu}]{Liu2010LRR1}
\bibinfo{author}{G.~Liu}, \bibinfo{author}{Z.~Lin}, \bibinfo{author}{Y.~Yu},
\newblock \bibinfo{title}{Robust subspace segmentation by low-rank
  representation},
\newblock \bibinfo{journal}{In ICML}  (\bibinfo{year}{2010}).
\bibitem[{Elhamifar and Vidal(2013)}]{Elhamifar2013SSC}
\bibinfo{author}{E.~Elhamifar}, \bibinfo{author}{R.~Vidal},
\newblock \bibinfo{title}{Sparse subspace clustering algorithm, theory, and
  applications},
\newblock \bibinfo{journal}{IEEE Trans. Pattern Anal. and Mach. Intell.}
  \bibinfo{volume}{35} (\bibinfo{year}{2013}) \bibinfo{pages}{2765--2781}.
\bibitem[{Rao et~al.(2008)Rao, Tron, Vidala, and Ma}]{Rao2008MS}
\bibinfo{author}{S.~Rao}, \bibinfo{author}{R.~Tron},
  \bibinfo{author}{R.~Vidala}, \bibinfo{author}{Y.~Ma},
\newblock \bibinfo{title}{Motion segmentation via robust subspace separation in
  the presence of outlying, incomplete, or corrupted trajectories},
\newblock \bibinfo{journal}{In CVPR}  (\bibinfo{year}{2008}).
\bibitem[{Lauer and Schn$\ddot{o}$rr(2009)}]{Lauer2009SC}
\bibinfo{author}{F.~Lauer}, \bibinfo{author}{C.~Schn$\ddot{o}$rr},
\newblock \bibinfo{title}{Spectral clustering of linear subspaces for motion
  segmentation},
\newblock \bibinfo{journal}{in IEEE International Conference on Computer
  Vision}  (\bibinfo{year}{2009}) \bibinfo{pages}{678--685}.
\bibitem[{Rao et~al.(2010)Rao, Tron, Vidal, and Ma}]{Rao2010MotionSeg}
\bibinfo{author}{S.~Rao}, \bibinfo{author}{R.~Tron},
  \bibinfo{author}{R.~Vidal}, \bibinfo{author}{Y.~Ma},
\newblock \bibinfo{title}{Motion segmentation in the presence of outlying,
  incomplete, or corrupted trajectories},
\newblock \bibinfo{journal}{IEEE Trans. Pattern Anal. and Mach. Intell.}
  \bibinfo{volume}{32} (\bibinfo{year}{2010}) \bibinfo{pages}{1832--1845}.
\bibitem[{Aldroubi and Sekmen(2012)}]{Aldroubi2012MS}
\bibinfo{author}{A.~Aldroubi}, \bibinfo{author}{A.~Sekmen},
\newblock \bibinfo{title}{Nearness to local subspace algorithm for subspace and
  motion segmentation},
\newblock \bibinfo{journal}{IEEE Signal Processing Letters}
  \bibinfo{volume}{19} (\bibinfo{year}{2012}) \bibinfo{pages}{704--707}.
\bibitem[{Vidala and Favarob(2013)}]{Vidala2013LRSC}
\bibinfo{author}{R.~Vidala}, \bibinfo{author}{P.~Favarob},
\newblock \bibinfo{title}{Low rank subspace clustering {(LRSC)}},
\newblock \bibinfo{journal}{Pattern Recognition Letters}
  (\bibinfo{year}{2013}).
\bibitem[{Vidal(2010)}]{Vidal2010SC}
\bibinfo{author}{R.~Vidal},
\newblock \bibinfo{title}{A tutorial on subspace clustering},
\newblock \bibinfo{journal}{IEEE Signal Processing Magazine}
  \bibinfo{volume}{28} (\bibinfo{year}{2010}) \bibinfo{pages}{52--68}.
\bibitem[{Sim et~al.(2013)Sim, Gopalkrishnan, Zimek, and Cong}]{Sim2013SC}
\bibinfo{author}{K.~Sim}, \bibinfo{author}{V.~Gopalkrishnan},
  \bibinfo{author}{A.~Zimek}, \bibinfo{author}{G.~Cong},
\newblock \bibinfo{title}{A survey on enhanced subspace clustering},
\newblock \bibinfo{journal}{Data mining and knowledge discovery}
  \bibinfo{volume}{26} (\bibinfo{year}{2013}) \bibinfo{pages}{332--397}.
\bibitem[{McWilliams and Montana(2014)}]{BMcWilliams2014SC}
\bibinfo{author}{B.~McWilliams}, \bibinfo{author}{G.~Montana},
\newblock \bibinfo{title}{Subspace clustering of high-dimensional data: a
  predictive approach},
\newblock \bibinfo{journal}{Data Mining and Knowledge Discovery}
  \bibinfo{volume}{28} (\bibinfo{year}{2014}) \bibinfo{pages}{736--772}.
\bibitem[{Boult and Brown(1991)}]{Boult1991Factor}
\bibinfo{author}{T.~Boult}, \bibinfo{author}{L.~Brown},
\newblock \bibinfo{title}{Factorization-based segmentation of motions},
\newblock \bibinfo{journal}{in IEEE Workshop on Proceedings of the Visual
  Motion}  (\bibinfo{year}{1991}) \bibinfo{pages}{179--186}.
\bibitem[{Basri and Jacobs(2003)}]{Basri2003}
\bibinfo{author}{R.~Basri}, \bibinfo{author}{D.~W. Jacobs},
\newblock \bibinfo{title}{Lambertian reflectance and linear subspaces},
\newblock \bibinfo{journal}{IEEE Trans. Pattern Anal. and Mach. Intell.}
  \bibinfo{volume}{25} (\bibinfo{year}{2003}) \bibinfo{pages}{218--233}.
\bibitem[{Dyer et~al.(2013)Dyer, Sankaranarayanan, and Baraniuk}]{Dyer2013SC}
\bibinfo{author}{E.~Dyer}, \bibinfo{author}{A.~Sankaranarayanan},
  \bibinfo{author}{R.~Baraniuk},
\newblock \bibinfo{title}{Greedy feature selection for subspace clustering},
\newblock \bibinfo{journal}{The Journal of Machine Learning Research}
  \bibinfo{volume}{14} (\bibinfo{year}{2013}) \bibinfo{pages}{2487--2517}.
\bibitem[{Vidal et~al.(2005)Vidal, Ma, and Sastry}]{Vidal2005GPCA}
\bibinfo{author}{R.~Vidal}, \bibinfo{author}{Y.~Ma},
  \bibinfo{author}{S.~Sastry},
\newblock \bibinfo{title}{Generalized principal component analysis {(GPCA)}},
\newblock \bibinfo{journal}{IEEE Trans. Pattern Anal. and Mach. Intell.}
  \bibinfo{volume}{27} (\bibinfo{year}{2005}) \bibinfo{pages}{1945--1959}.
\bibitem[{Fischler and Bolles(1981)}]{Fischler1981RANSAC}
\bibinfo{author}{M.~Fischler}, \bibinfo{author}{R.~Bolles},
\newblock \bibinfo{title}{Random sample consensus: a paradigm for model fitting
  with applications to image analysis and automated cartography},
\newblock \bibinfo{journal}{Communications of the ACM} \bibinfo{volume}{24}
  (\bibinfo{year}{1981}) \bibinfo{pages}{381--395}.
\bibitem[{Ho et~al.(2003)Ho, Lim, Lee, and Kriegman}]{Ho2003KSC}
\bibinfo{author}{J.~Ho}, \bibinfo{author}{M.~Y.~J. Lim},
  \bibinfo{author}{K.~Lee}, \bibinfo{author}{D.~Kriegman},
\newblock \bibinfo{title}{Clustering appearances of objects under varying
  illumination conditions},
\newblock \bibinfo{journal}{In CVPR}  (\bibinfo{year}{2003}).
\bibitem[{Ni et~al.(2010)Ni, Sun, X.J.~Yuan, and Chong}]{Ni2010LRRPSD}
\bibinfo{author}{Y.~Ni}, \bibinfo{author}{J.~Sun}, \bibinfo{author}{S.~Y.
  X.J.~Yuan}, \bibinfo{author}{L.~Chong},
\newblock \bibinfo{title}{Robust low-rank subspace segmentation with
  semidefinite guarantees},
\newblock \bibinfo{journal}{Proceedings of the IEEE International Conference on
  Data Mining Workshops (ICDMW)}  (\bibinfo{year}{2010})
  \bibinfo{pages}{1179--1188}.
\bibitem[{L.Zhuang et~al.(2012)L.Zhuang, Gao, Lin, Ma, Zhang, and
  Yu}]{Zhuang2012NLRR}
\bibinfo{author}{L.Zhuang}, \bibinfo{author}{H.~Gao}, \bibinfo{author}{Z.~Lin},
  \bibinfo{author}{Y.~Ma}, \bibinfo{author}{X.~Zhang}, \bibinfo{author}{N.~Yu},
\newblock \bibinfo{title}{Non-negative low rank and sparse graph for
  semi-supervised learning},
\newblock \bibinfo{journal}{In CVPR}  (\bibinfo{year}{2012}).
\bibitem[{Peng et~al.(2012)Peng, Zhang, and Yi}]{Peng2012L2Graph}
\bibinfo{author}{X.~Peng}, \bibinfo{author}{L.~Zhang}, \bibinfo{author}{Z.~Yi},
\newblock \bibinfo{title}{Constructing the l2-graph for subspace learning and
  subspace clustering},
\newblock \bibinfo{journal}{arXiv preprint arXiv:1209.0841}
  (\bibinfo{year}{2012}).
\bibitem[{Costeira and Kanade(1998)}]{Costeira1998MF}
\bibinfo{author}{J.~P. Costeira}, \bibinfo{author}{T.~Kanade},
\newblock \bibinfo{title}{A multibody factorization method for independently
  moving objects},
\newblock \bibinfo{journal}{Int. J. Comput. Vision} \bibinfo{volume}{29}
  (\bibinfo{year}{1998}) \bibinfo{pages}{159--179}.
\bibitem[{Wei and Lin(2011)}]{Wei2011RSI}
\bibinfo{author}{S.~Wei}, \bibinfo{author}{Z.~Lin},
\newblock \bibinfo{title}{Analysis and improvement of low rank representation
  for subspace segmentation},
\newblock \bibinfo{journal}{arXiv:1107.1561}  (\bibinfo{year}{2011}).
\bibitem[{Liu et~al.(2012)Liu, Lin, Torre, and Su}]{Liu2012FRR}
\bibinfo{author}{R.~Liu}, \bibinfo{author}{Z.~Lin}, \bibinfo{author}{F.~Torre},
  \bibinfo{author}{Z.~Su},
\newblock \bibinfo{title}{Fixed-rank representation for unsupervised visual
  learning},
\newblock \bibinfo{journal}{In CVPR}  (\bibinfo{year}{2012}).
\bibitem[{Liu et~al.(2013)Liu, Jiao, and Shang}]{Liu2013MF}
\bibinfo{author}{Y.~Liu}, \bibinfo{author}{L.~Jiao},
  \bibinfo{author}{F.~Shang},
\newblock \bibinfo{title}{An efficient matrix factorization based low-rank
  representation for subspace clustering},
\newblock \bibinfo{journal}{Pattern Recognition} \bibinfo{volume}{46}
  (\bibinfo{year}{2013}) \bibinfo{pages}{284--292}.
\bibitem[{Zhang et~al.(2014)Zhang, Yi, and Peng}]{Zhang2014FLRR}
\bibinfo{author}{H.~Zhang}, \bibinfo{author}{Z.~Yi}, \bibinfo{author}{X.~Peng},
\newblock \bibinfo{title}{flrr: fast low-rank representation using
  frobenius-norm},
\newblock \bibinfo{journal}{Electronics Letters} \bibinfo{volume}{50}
  (\bibinfo{year}{2014}) \bibinfo{pages}{936--938}.
\bibitem[{Luxburg(2007)}]{Luxburg2007SC}
\bibinfo{author}{U.~V. Luxburg},
\newblock \bibinfo{title}{A tutorial on spectral clustering},
\newblock \bibinfo{journal}{Statistics and computing} \bibinfo{volume}{17}
  (\bibinfo{year}{2007}) \bibinfo{pages}{395--416}.
\bibitem[{Shi et~al.(2000)Shi, Malik, and S.Sastry}]{Shi2000Ncuts}
\bibinfo{author}{J.~Shi}, \bibinfo{author}{J.~Malik},
  \bibinfo{author}{S.Sastry},
\newblock \bibinfo{title}{Normalized cuts and image segmentation},
\newblock \bibinfo{journal}{IEEE Trans. Pattern Anal. and Mach. Intell.}
  \bibinfo{volume}{22} (\bibinfo{year}{2000}) \bibinfo{pages}{888--905}.
\bibitem[{Tibshiran(1996)}]{Tibshiran1996Lasso}
\bibinfo{author}{R.~Tibshiran},
\newblock \bibinfo{title}{Regression shrinkage and selection via the lasso},
\newblock \bibinfo{journal}{Journal of the Royal Statistical Society Series B}
  \bibinfo{volume}{58} (\bibinfo{year}{1996}) \bibinfo{pages}{267--288}.
\bibitem[{Donoho(2006)}]{Donoho2006MinimalL1Norm}
\bibinfo{author}{D.~Donoho},
\newblock \bibinfo{title}{For most large underdetermined systems of linear
  equations the minimal ${l_1}$-norm solution is also the sparsest solution},
\newblock \bibinfo{journal}{Comm. Pure and Applied Math.} \bibinfo{volume}{59}
  (\bibinfo{year}{2006}) \bibinfo{pages}{797--829}.
\bibitem[{Cand$\grave{e}$s et~al.(2008)Cand$\grave{e}$s, Wakin, and
  Boyd}]{Cand2008l1norm}
\bibinfo{author}{E.~J. Cand$\grave{e}$s}, \bibinfo{author}{M.~B. Wakin},
  \bibinfo{author}{S.~P. Boyd},
\newblock \bibinfo{title}{Enhancing sparsity by reweighted ${l_1}$
  minimization},
\newblock \bibinfo{journal}{Journal of Fourier Analysis and Applications}
  \bibinfo{volume}{14} (\bibinfo{year}{2008}) \bibinfo{pages}{877--905}.
\bibitem[{Nasihatkon and Hartley(2011)}]{Nasihatkon2011SSC}
\bibinfo{author}{B.~Nasihatkon}, \bibinfo{author}{R.~Hartley},
\newblock \bibinfo{title}{Graph connectivity in sparse subspace clustering},
\newblock \bibinfo{journal}{in CVPR}  (\bibinfo{year}{2011})
  \bibinfo{pages}{2137--2144}.
\bibitem[{Wang and Xu(2013)}]{Wang2013NSSC}
\bibinfo{author}{Y.~Wang}, \bibinfo{author}{H.~Xu},
\newblock \bibinfo{title}{Noisy sparse subspace clustering},
\newblock \bibinfo{journal}{arXiv preprint arXiv:1309.1233}
  (\bibinfo{year}{2013}).
\bibitem[{Bao et~al.(2012)Bao, Liu, Xu, and Yan}]{Bao2012IRPCA}
\bibinfo{author}{B.~Bao}, \bibinfo{author}{G.~Liu}, \bibinfo{author}{C.~Xu},
  \bibinfo{author}{S.~Yan},
\newblock \bibinfo{title}{Inductive robust principal component analysis},
\newblock \bibinfo{journal}{IEEE Trans. Image Processing} \bibinfo{volume}{21}
  (\bibinfo{year}{2012}) \bibinfo{pages}{3794--3800}.
\bibitem[{Z.~Lin and Ma(2010)}]{Lin2010ALM}
\bibinfo{author}{M.~C. Z.~Lin}, \bibinfo{author}{Y.~Ma},
\newblock \bibinfo{title}{The augmented lagrange multiplier method for exact
  recovery of corrupted low-rank matrices},
\newblock \bibinfo{journal}{arXiv preprint arXiv:1009.5055}
  (\bibinfo{year}{2010}).
\bibitem[{Lin et~al.(2011)Lin, Liu, and Su}]{Lin2011LADM}
\bibinfo{author}{Z.~Lin}, \bibinfo{author}{R.~Liu}, \bibinfo{author}{Z.~Su},
\newblock \bibinfo{title}{Linearized alternating direction method with adaptive
  penalty for low-rank representation},
\newblock \bibinfo{journal}{In NIPS}  (\bibinfo{year}{2011}).
\bibitem[{Chen and Yi(2014)}]{Chen2014SC}
\bibinfo{author}{J.~Chen}, \bibinfo{author}{Z.~Yi},
\newblock \bibinfo{title}{Subspace clustering by exploiting a low-rank
  representation with a symmetric constraint},
\newblock \bibinfo{journal}{arXiv preprint arXiv:1403.2330}
  (\bibinfo{year}{2014}).
\bibitem[{Zhang et~al.(2011)Zhang, Yang, and Feng}]{Zhang2011SRCR}
\bibinfo{author}{L.~Zhang}, \bibinfo{author}{M.~Yang},
  \bibinfo{author}{X.~Feng},
\newblock \bibinfo{title}{Sparse representation or collaborative
  representation: Which helps face recognition?},
\newblock \bibinfo{journal}{In ICCV}  (\bibinfo{year}{2011})
  \bibinfo{pages}{471--478}.
\bibitem[{Lu et~al.(2012)Lu, Min, Zhao, Zhu, Huang, and Yan}]{Lu2012LSR}
\bibinfo{author}{C.~Lu}, \bibinfo{author}{H.~Min}, \bibinfo{author}{Z.~Zhao},
  \bibinfo{author}{L.~Zhu}, \bibinfo{author}{D.~Huang},
  \bibinfo{author}{S.~Yan},
\newblock \bibinfo{title}{Robust and efficient subspace segmentation via least
  squares regression},
\newblock \bibinfo{journal}{In ICCV}  (\bibinfo{year}{2012})
  \bibinfo{pages}{347--360}.
\bibitem[{Wang et~al.(2011)Wang, Yuan, Yao, Yan, and Shen}]{Wang2011ESS}
\bibinfo{author}{S.~Wang}, \bibinfo{author}{X.~Yuan}, \bibinfo{author}{T.~Yao},
  \bibinfo{author}{S.~Yan}, \bibinfo{author}{J.~Shen},
\newblock \bibinfo{title}{Efficient subspace segmentation via quadratic
  programming},
\newblock \bibinfo{journal}{In AAAI}  (\bibinfo{year}{2011}).
\bibitem[{Jolliffe(2002)}]{Jolliffe2002PCA}
\bibinfo{author}{I.~Jolliffe}, \bibinfo{title}{Principal Component Analysis},
  \bibinfo{publisher}{Springer New York}, \bibinfo{year}{2002}.
\bibitem[{Wright et~al.(2009)Wright, Ganesh, Rao, Peng, and
  Ma}]{Wright2009RPCA}
\bibinfo{author}{J.~Wright}, \bibinfo{author}{A.~Ganesh},
  \bibinfo{author}{S.~Rao}, \bibinfo{author}{Y.~Peng}, \bibinfo{author}{Y.~Ma},
\newblock \bibinfo{title}{Robust principal component analysis: Exact recovery
  of corrupted low-rank matrices by convex optimization},
\newblock \bibinfo{journal}{In NIPS}  (\bibinfo{year}{2009})
  \bibinfo{pages}{2080--2088}.
\bibitem[{Cand$\grave{e}$s et~al.(2011)Cand$\grave{e}$s, Li, Ma, and
  Wright}]{Candes2011RobustPCA}
\bibinfo{author}{E.~J. Cand$\grave{e}$s}, \bibinfo{author}{X.~Li},
  \bibinfo{author}{Y.~Ma}, \bibinfo{author}{J.~Wright},
\newblock \bibinfo{title}{Robust principal component analysis},
\newblock \bibinfo{journal}{Journal of the ACM (JACM)} \bibinfo{volume}{58}
  (\bibinfo{year}{2011}).
\bibitem[{Kaski(1998)}]{Kaski1998RM}
\bibinfo{author}{S.~Kaski},
\newblock \bibinfo{title}{Dimensionality reduction by random mapping: Fast
  similarity computation for clustering},
\newblock \bibinfo{journal}{IEEE International Joint Conf. Neural Networks
  Proceedings}  (\bibinfo{year}{1998}) \bibinfo{pages}{413--418}.
\bibitem[{Bingham and Mannila(2001)}]{Bingham2001RP}
\bibinfo{author}{E.~Bingham}, \bibinfo{author}{H.~Mannila},
\newblock \bibinfo{title}{Random projection in dimensionality reduction:
  applications to image and text data},
\newblock \bibinfo{journal}{Proc. ACM SIGKDD international Conf. Knowledge
  discovery and data mining}  (\bibinfo{year}{2001}) \bibinfo{pages}{245--250}.
\bibitem[{Lee et~al.(2005)Lee, Ho, and Kriegman}]{Lee05YALEB}
\bibinfo{author}{K.~Lee}, \bibinfo{author}{J.~Ho},
  \bibinfo{author}{D.~Kriegman},
\newblock \bibinfo{title}{Acquiring linear subspaces for face recognition under
  variable lighting},
\newblock \bibinfo{journal}{IEEE Trans. Pattern Anal. and Mach. Intell.}
  \bibinfo{volume}{27} (\bibinfo{year}{2005}) \bibinfo{pages}{684--698}.
\bibitem[{Georghiades et~al.(2011)Georghiades, Belhumeur, and
  Kriegman}]{GeBeKr01YALEB}
\bibinfo{author}{A.~Georghiades}, \bibinfo{author}{P.~Belhumeur},
  \bibinfo{author}{D.~Kriegman},
\newblock \bibinfo{title}{From few to many: Illumination cone models for face
  recognition under variable lighting and pose},
\newblock \bibinfo{journal}{IEEE Trans. Pattern Anal. and Mach. Intell.}
  \bibinfo{volume}{23} (\bibinfo{year}{2011}) \bibinfo{pages}{643--660}.
\bibitem[{Tron and Vidal(2007)}]{Hopkins155}
\bibinfo{author}{R.~Tron}, \bibinfo{author}{R.~Vidal},
\newblock \bibinfo{title}{A benchmark for the comparison of 3-d motion
  segmentation algorithms},
\newblock \bibinfo{journal}{In CVPR}  (\bibinfo{year}{2007}).
\bibitem[{Yan and Pollefeys(2006)}]{Yang2006MotionSeg}
\bibinfo{author}{J.~Yan}, \bibinfo{author}{M.~Pollefeys},
\newblock \bibinfo{title}{A general framework for motion segmentation:
  Independent, articulated, rigid, non-rigid, degenerate and non-degenerate},
\newblock \bibinfo{journal}{In European Conf. on Computer Vision}
  (\bibinfo{year}{2006}) \bibinfo{pages}{94--106}.

\end{thebibliography}

\end{document}